\documentclass[journal]{IEEEtran}
\usepackage{cite}
\usepackage[cmex10]{amsmath}
\usepackage{comment}
\usepackage{graphicx}
\usepackage{amsmath, amssymb}
\usepackage{amsfonts}
\usepackage{ascmac}
\usepackage{subcaption}
\usepackage{bm}
\usepackage{url}
\usepackage{authblk}
\usepackage{amsthm}
\usepackage{empheq}

\usepackage[
  colorlinks=true,
  linkcolor=blue,
  urlcolor=blue,
  citecolor=blue
]{hyperref}

\newcommand{\ali}[1]{\begin{align} #1 \end{align}}
\newcommand{\nn}{\nonumber}

\newcommand{\mat}[1]{\begin{bmatrix} #1 \end{bmatrix}}
\newcommand{\md}{\mathop{}\!\mathrm{d}}
\newcommand{\me}{\mathrm{e}}

\newcommand{\pa}{\mathop{}\!\partial}

\newtheorem{theorem}{Theorem}
\newtheorem{lemma}{Lemma}

\newtheorem{proposition}{Proposition}
\usepackage{mdframed}
\newmdtheoremenv{theorembox}{Theorem}

\begin{document}
\title{Wolkowicz-Styan Upper Bound on the Hessian Eigenspectrum for Cross-Entropy Loss in Nonlinear Smooth Neural Networks}

\author{
Yuto Omae$^1$, Kazuki Sakai$^2$, Yohei Kakimoto$^1$,
Makoto Sasaki$^1$, Yusuke Sakai$^3$, Hirotaka Takahashi$^3$
\thanks{1: Nihon University (Japan)}
\thanks{2: National Institute of Technology, Nagaoka College (Japan)}
\thanks{3: Tokyo City University (Japan)}
}

\markboth{Preprint}%
{Y. Omae \MakeLowercase{\textit{et al.}}}
\maketitle

\begin{abstract}
Neural networks (NNs) constitute the core of contemporary machine learning, delivering state-of-the-art results in diverse applications.
However, a comprehensive theoretical understanding of the relationship between the geometry of the loss function and generalization capabilities is still lacking.
The local geometry of the loss function near a critical point is well-approximated by its quadratic form, obtained through a second-order Taylor expansion.
The coefficients of the quadratic term correspond to the Hessian matrix, whose eigenspectrum allows us to evaluate the sharpness of the loss at the critical point.
Extensive research has suggested that flat critical points are often associated with better generalization performance, whereas sharp critical points tend to be linked to higher generalization error.
While quantifying sharpness requires the eigenspectrum of the Hessian, the characteristic equation of a general square matrix does not admit a closed-form solution in general.
Therefore, most existing studies on evaluating loss sharpness rely on numerical approximation methods.
Existing closed-form analyses of the eigenspectrum are primarily limited to simplified architectures, such as linear or ReLU-activated networks; consequently, theoretical analysis of smooth nonlinear multilayer neural networks remains limited.
Against this background, this study focuses on nonlinear, smooth multilayer neural networks and derives a closed-form upper bound for the maximum eigenvalue of the Hessian with respect to the cross-entropy loss by leveraging the Wolkowicz-Styan bound (Wolkowicz and Styan 1980).
Specifically, the derived upper bound is expressed as a function of the affine transformation parameters, hidden layer dimensions, and the degree of orthogonality among the training samples.
The primary contribution of this paper is an analytical characterization of loss sharpness in smooth nonlinear multilayer neural networks via a closed-form expression, avoiding explicit numerical eigenspectrum computation.
We hope that this work provides a small yet meaningful step toward unraveling the mysteries of deep learning.
\end{abstract}


\section{Introduction}
Neural Networks (NNs) serve as a cornerstone of modern machine learning, delivering state-of-the-art results in various domains, including image recognition \cite{chaiDeep2021}, natural language processing \cite{arkhangelskayaDeep2023}, and speech recognition \cite{mehrishReview2023}.
However, despite their remarkable success, the fundamental reasons why NNs exhibit such high generalization performance remain largely unexplained.
In this study, we focus on the relationship between the critical points of the loss function and generalization performance to better understand generalization.
In previous studies, it has been posited that the sharpness of a critical point is closely related to generalization performance \cite{dinhSharp2017, yueSALR2024}.
Specifically, flat minima have often been associated with better generalization, whereas sharp minima are linked to poorer generalization performance, as the loss is more sensitive to small parameter perturbations~\cite{hochreiterFlat1997}.
This sharpness is characterized by the quadratic term of the Taylor expansion of the loss function around a critical point \cite{liuHessian2023}.
Since the eigenvalues of the Hessian represent the curvature of the loss surface, they serve as a metric for sharpness \cite{aroraUnderstanding2022}.
In particular, the maximum eigenvalue $\lambda_1$ is employed as a key indicator of the curvature of the loss landscape \cite{lyuUnderstanding2023}.
As discussed above, the analysis of the Hessian eigenspectrum has played a central role in understanding learning dynamics.

However, since the size of the Hessian in NNs is often prohibitively large, providing a rigorous analysis of its eigenspectrum is generally intractable.
Specifically, beyond a certain dimensionality, the characteristic equation of the Hessian does not admit a closed-form solution in general, precluding an exact analytical expression for its eigenvalues.
For this reason, numerical approximation of the eigenspectrum has been widely adopted as a practical and effective strategy in the literature.
Specific techniques include the Lanczos method \cite{lanczosIteration1950}, which provides an approximate solution for the eigenspectrum of Hermitian matrices, and the Hutchinson method \cite{hutchinsonStochastic1989}, an algorithm for approximating the trace of square matrices.
Indeed, many existing works rely on these numerical techniques for sharpness evaluation; representative examples include \cite{yaoPyHessian2020, ghorbaniInvestigation2019} for the Lanczos iteration and \cite{liuHessian2023, dongHAWQV22020} for the Hutchinson method.
Despite their practical utility, numerical approaches lack the capacity to provide an analytical link between sharpness, training data, and model parameters, which remains a key bottleneck for theoretical understanding.

Against this background, Singh et al. \cite{singhCracking2025} derived closed-form expressions for the eigenvalues of the loss Hessian in linear networks, including those with identity and ReLU activations.
This work represents a pioneering effort that provides a foundational framework for closed-form Hessian analysis.
However, research on closed-form expressions for the eigenspectrum in nonlinear, smooth neural networks, which are widely used in practical deep learning, remains insufficient.
In order to address this limitation, we focus on multilayer NNs with nonlinear smooth activations and pursue a closed-form characterization of the maximum Hessian eigenvalue $\lambda_1$ for cross-entropy loss functions.
Since characteristic equations of degree five or higher do not admit closed-form solutions in general, obtaining an explicit closed-form representation of Hessian eigenvalues for high-dimensional matrices is typically intractable.
However, a known theorem provides an analytical upper bound $\lambda_{\sup}(\bm{\theta})$ for the maximum eigenvalue by utilizing the traces of both the matrix and its square \cite{wolkowiczBounds1980}.
This upper bound is frequently referred to as the Wolkowicz-Styan bound \cite{merikoskiBounds1997, merikoskiBest2001, sharmaNote2014}.
By applying this theorem to the loss Hessian, we obtain an analytical closed-form representation of the upper bound $\lambda_{\mathrm{sup}}(\bm{\theta})$ for its maximum eigenvalue.
Our analysis covers traditional Sigmoid \cite{minaiDerivatives1993} and hyperbolic tangent (Tanh) \cite{goodfellowDeep2016} activations, alongside modern SoftPlus \cite{berzalDL1012025} and GELU \cite{hendrycksGaussian2016} activations.
The case of the linear activation is also included as a foundational reference to provide a clear starting point for our theoretical discussion.

By examining the derived upper bound $\lambda_{\mathrm{sup}}(\bm{\theta})$, we reveal that sharpness is determined by factors such as the norms of parameters in affine transformations, the dimensionality of hidden layers, and the degree of orthogonality among the training samples.
These findings provide an analytical characterization of loss sharpness, showing that it is governed by factors such as the network architecture and the underlying training data.
These findings constitute the primary scholarly contribution of this paper, providing a novel theoretical perspective for understanding generalization in neural networks.

\section{Related works}
In previous studies, it has been widely recognized that the sharpness of a critical point reached during training is closely linked to generalization; specifically, flat minima tend to yield high generalization performance, whereas sharp minima are associated with larger generalization errors \cite{liVisualizing2018}.
From this perspective, several empirical techniques widely adopted for performance enhancement have been linked to the reduction of sharpness.
For instance, Ghorbani et al.~\cite{ghorbaniInvestigation2019} and Lyu et al.~\cite{lyuUnderstanding2023} observed that batch normalization effectively diminishes loss sharpness.
Similarly, Wei et al.~\cite{weiHow2019} and Wu et al.~\cite{wuImplicit2023} demonstrated that Stochastic Gradient Descent (SGD), a cornerstone of modern optimization, contributes to sharpness reduction.
Furthermore, Li et al.~\cite{liVisualizing2018} showed that skip connections, such as those used in ResNet architectures, have the effect of smoothing the loss landscape and reducing its sharpness.
Moreover, a specific study has analyzed the relationship between the learning rate and sharpness.
For instance, Marion et al.~\cite{marionDeep2024} reported that minima obtained through gradient descent in the limit where the learning rate approaches zero tend to be flat.
Additionally, a learning algorithm called Sharpness-Aware Learning Rate Scheduler (SALRS) has been proposed to reach flat minima by dynamically adjusting the learning rate \cite{yueSALR2024}.

The eigenvalues of the Hessian are frequently employed as a metric for quantifying sharpness \cite{liuHessian2023, aroraUnderstanding2022}.
However, obtaining the Hessian eigenvalues analytically is often intractable, and their numerical approximation requires significant computational resources.
Accordingly, recent works have focused on the Hessian trace as a more accessible proxy for curvature, bypassing the high costs associated with eigenvalue estimation \cite{damianLabel2021, liuSame2023}.
Since the trace is the sum of the eigenvalues, it serves as a valid indicator of sharpness in cases where the Hessian is positive definite.
The Hessian trace can be efficiently estimated at a low computational cost using Hutchinson's method \cite{hutchinsonStochastic1989}.
Leveraging this efficiency, Liu et al. \cite{liuHessian2023} proposed a regularization technique that directly penalizes the Hessian trace during training.
Furthermore, Sankar et al.~\cite{sankarDeeper2021} discovered that the Hessian eigenvalues of intermediate layers exhibit a strong similarity to those of the entire model.
Based on this insight, they proposed a novel regularization technique termed Layer-wise Hessian Trace Regularization.
Similarly, Bolshim et al.~\cite{bolshimLocal2025} reported that the layer-wise Hessian eigenspectrum is closely associated with generalization performance.
Furthermore, recent studies have proposed noise-injection techniques to guide parameters toward flat minima \cite{zhangNoise2024}, as well as relocation algorithms that shift parameters into flatter regions without compromising training performance \cite{zhouImproving2025}.
Although the Hessian trace is widely adopted as a proxy for sharpness in existing literature, the minimization of the maximum eigenvalue remains a more pivotal factor in improving a model's generalization capability.
Accordingly, an algorithm that directly targets the maximum eigenvalue for regularization has been developed \cite{luoExplicit2025}.

Although a specific algorithm for the exact computation of the Hessian in neural networks has been available since early research \cite{bishopExact1992}, its computational cost is substantial.
Consequently, learning algorithms designed to reduce sharpness without explicitly calculating the Hessian have also been proposed.
A prominent example of such an approach is Sharpness-Aware Minimization (SAM), proposed by Foret et al.~\cite{foretSharpnessAware2021}.
Although SAM does not explicitly target the reduction of the maximum eigenvalue, it has been observed to effectively decrease its value in practice.
As theoretical advancements of SAM, several variants have been proposed, such as ASAM~\cite{kwonASAM2021}, which overcomes the sensitivity to weight parameter scaling, and ImbSAM~\cite{zhouImbSAM2023a}, which functions effectively even with imbalanced data.
Significant efforts are also underway to theoretically elucidate the effectiveness of SAM \cite{andriushchenkoUnderstanding2022}.
In addition to theoretical investigations, SAM has also been reported to enhance test performance across various applied tasks \cite{chenWhen2022, huangSharpness2025}.

These studies offer highly valuable insights from the perspective of elucidating the relationship between sharpness and generalization performance.
However, many of these approaches rely on numerical approximations and do not readily yield explicit analytical characterizations.
As also pointed out by Singh et al.~\cite{singhCracking2025}, the perspective of evaluating sharpness analytically, rather than through numerical approximation, is also important.
From this perspective, Singh et al. derived a closed-form expression for the rank of the Hessian in networks with linear activations \cite{singhAnalytic2021}.
This revealed the number of eigenvectors associated with non-zero curvature.
Furthermore, in subsequent research, closed-form expressions for the eigenvalues of the Hessian in linear or ReLU networks were derived \cite{singhCracking2025}.
This has made it possible to analytically examine the factors that increase sharpness.
However, a closed-form expression for the maximum eigenvalue in three-layer neural networks with smooth nonlinearities remains unavailable.
In this work, we address this gap by deriving a closed-form upper bound for the maximum Hessian eigenvalue.

\begin{figure}[t] 
    \centering
    \includegraphics[scale=0.26]{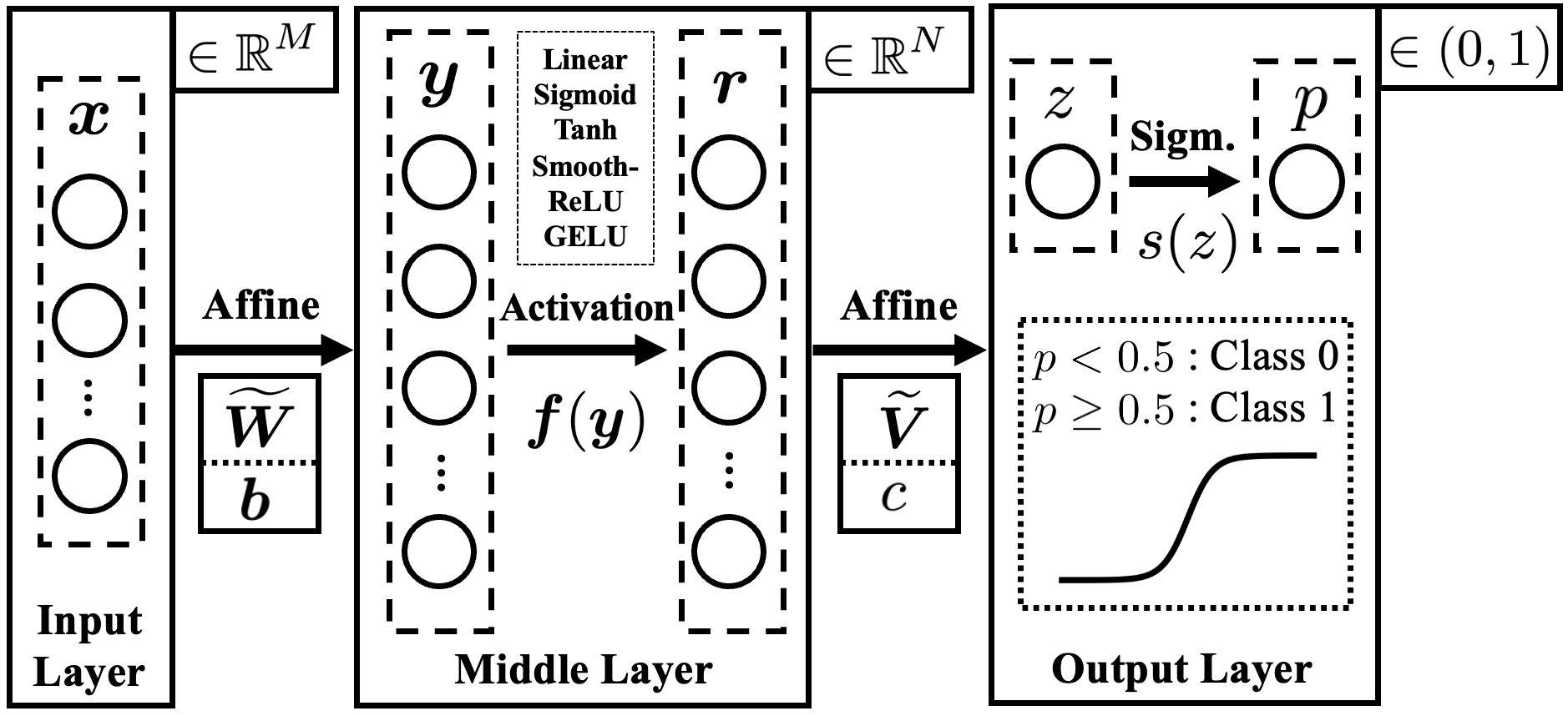}
\caption{The NN architecture analyzed in this study.}
    \label{fig_nn}
\end{figure}

\section{Model and Loss Function}
\subsection{3-Layer Neural Network}
Figure~\ref{fig_nn} illustrates the three-layer feedforward neural network analyzed in this study. 
The input to the network is denoted by $\bm{x}$, and the output is $p \in (0, 1)$.
This network is designed to solve a binary classification problem, where an input is classified as Class 1 if $p \geq 0.5$ and as Class 0 if $p < 0.5$.
The relationship between each variable is given by
\ali{
&\bm{y} = \widetilde{\bm{W}} \bm{x} + \bm{b}, \ 
\bm{r} = \bm{f}(\bm{y}), \ 
z = \widetilde{\bm{V}} \bm{r} + c, \
p = s(z), \nn \\
&\bm{x} \in \mathbb{R}^M, \ 
\widetilde{\bm{W}} \in \mathbb{R}^{N \times M}, \ 
\bm{b} \in \mathbb{R}^N, \ 
\bm{y} \in \mathbb{R}^N, \nn \\
& \bm{r} \in \mathbb{R}^N, \ 
\widetilde{\bm{V}} \in \mathbb{R}^{1 \times N}, \ 
c \in \mathbb{R}, \ 
z \in \mathbb{R}, \ 
p \in (0, 1). \label{eq_3nn}
}
The function $\bm{f}$ represents a vector-valued activation function; for further details, please refer to Appendix~\ref{sec_activation_base}.
The activation functions $\bm{f}$ for which we derive closed-form expressions in this paper include the linear activation, Sigmoid~\cite{minaiDerivatives1993}, Tanh~\cite{goodfellowDeep2016}, SmoothReLU (Softplus)~\cite{berzalDL1012025}, and GELU~\cite{hendrycksGaussian2016}. 
For their specific details, please refer to Appendices~\ref{sec_lin_details}, \ref{sec_sigm_details}, \ref{sec_tanh_details}, \ref{sec_srelu_details}, and \ref{sec_gelu_details}.
Note that $s(z)$ denotes the sigmoid function used for binary classification.
This model involves four sets of parameters: two weight matrices, $\widetilde{\bm{W}}$ and $\widetilde{\bm{V}}$, and two bias terms, $\bm{b}$ and $c$, which increases the complexity of the analysis.
To simplify the formulation, we absorb the bias terms into the weight matrices.
That is, 
\ali{
\bm{W} = \mat{\bm{b} & \widetilde{\bm{W}}} \in \mathbb{R}^{N \times (M+1)}, \
\bm{V} = \mat{c & \widetilde{\bm{V}}} \in \mathbb{R}^{1 \times (N+1)}. \nn
}
By defining $\bm{b} = \bm{W}_{:0}$, $\widetilde{\bm{W}} = \mat{\bm{W}_{:1} \cdots \bm{W}_{:M}}$, $c = v_0$, and $\widetilde{\bm{V}} = \mat{v_1 \cdots v_N}$, we obtain $\bm{W} = \mat{\bm{W}_{:0} \cdots \bm{W}_{:M}}$ and $\bm{V} = \mat{v_0 \cdots v_N}$.
Here, $\bm{W}_{:m}$ denotes a column vector consisting of the $m$-th column of $\bm{W}$.
The vectors transformed by the parameter matrices $\bm{W}$ and $\bm{V}$ are given by
\ali{
\bm{h}(\bm{x}) = \mat{1 \\ \bm{x}}  \in \mathbb{R}^{M+1}, \ 
\bm{h}(\bm{r}) = \mat{1 \\ \bm{r}}  \in \mathbb{R}^{N+1} \nn
}
where $\bm{h}$ is a function that prepends a 1 to the input vector at the 0-th dimension.
By letting $x_0 = r_0 = 1$, these vectors can also be expressed as 
$\bm{h}(\bm{x}) = \mat{x_0 & x_1 & \cdots & x_M}^\top$ and 
$\bm{h}(\bm{r}) = \mat{r_0 & r_1 & \cdots & r_N}^\top$.
Adopting these notations, the input-output relationship of the NN is expressed as
\ali{
\bm{y} = \bm{W} \bm{h}(\bm{x}), \ 
\bm{r} = \bm{f}(\bm{y}), \
z = \bm{V} \bm{h}(\bm{r}), \
p &= s(z). \label{eq_nn_forward}
}
The NN shown in Eq.~\eqref{eq_nn_forward} is a two-parameter model, which is equivalent to the four-parameter model presented in Eq.~\eqref{eq_3nn}.
We define the vectors $\bm{w}$ and $\bm{v}$ by vertically concatenating the elements of the parameter matrices $\bm{W}$ and $\bm{V}$ as
\ali{
\bm{w} &= \mat{\bm{W}_{:m}}_{m \in \{0, \cdots, M\}} \in \mathbb{R}^{(M+1)N \times 1}, \label{eq_w_v_flat} \\ 
\bm{v} &= \bm{V}^\top \in \mathbb{R}^{(N+1) \times 1}. \nn
}
Furthermore, the full parameter vector of the NN is defined by stacking $\bm{w}$ and $\bm{v}$, that is, 
\begin{align}
\bm{\theta} = \mat{\bm{w} \\ \bm{v}} \in \mathbb{R}^{D \times 1}, \ D = MN+2N+1. \nn
\end{align}

\subsection{Cross-Entropy Loss}
Since the output $p$ of the NN is a function of the input $\bm{x}$ and the parameters $\bm{\theta}$, it is denoted as $p(\bm{\theta}; \bm{x})$.
Given the ground-truth label $q \in \{0, 1\}$, the binary cross-entropy loss for a single data point is defined in \cite{suryadiJacobian2025} as
 \ali{
l(\bm{\theta} ; \bm{x}, q) = - q \log p(\bm{\theta}; \bm{x})  - (1-q) \log(1- p(\bm{\theta} ; \bm{x})). \label{eq_unit_ce}
}
The semicolon denotes the separation of variables and constants, signifying that $\bm{x}$ and $q$ are constants and $\bm{\theta}$ is a variable once the training data has been collected.
Hereafter, we will write $l(\bm{\theta} ; \bm{x}_i, q_i)$ as $l_i$ or $l_i(\bm{\theta})$.
We denote the gradient and the Hessian as
 \ali{
\frac{\pa l}{\pa \bm{\theta}} = \mat{\frac{\pa l}{\pa (\bm{\theta})_d}} \in \mathbb{R}^{D\times 1}, 
\bm{H}_{l}(\bm{\theta}, \bm{\theta}) = \frac{\pa }{\pa \bm{\theta}} \bigg( \frac{\pa l}{\pa \bm{\theta}} \bigg)^\top \in \mathbb{R}^{D \times D}, \nn
}
respectively.
Here, $(\bm{\theta})_d$ represents the $d$-th element of $\bm{\theta}$.
The cross-entropy loss for the entire training dataset is given by
 \ali{
L(\bm{\theta}) = \sum_{i=1}^{I} l_i(\bm{\theta}) \in (0, \infty) \nn
}
where $I$ is the total number of data points.
Due to the linearity of the derivative, the gradient and the Hessian of $L$ are given by
\ali{
\frac{\pa L}{\pa \bm{\theta}} = \sum_{i=1}^I \frac{\pa l_i}{\pa \bm{\theta}} \in \mathbb{R}^{D \times 1}, \ 
\bm{H}_L(\bm{\theta}, \bm{\theta}) = \sum_{i=1}^I  \bm{H}_{l_i}(\bm{\theta}, \bm{\theta}) \in \mathbb{R}^{D \times D}. \label{eq_hessian_L}
}

\begin{figure}[t] 
    \centering
    \includegraphics[scale=0.8]{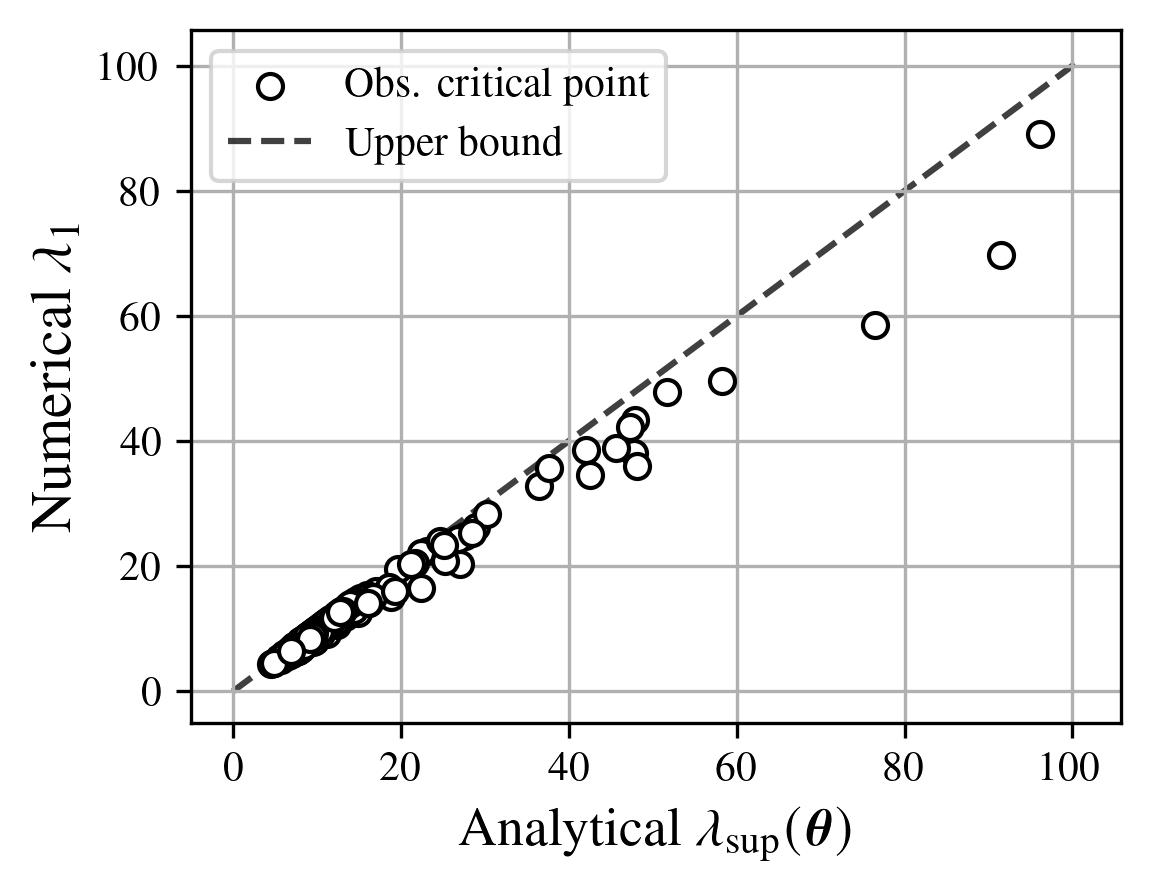}
    \caption{
    Relationship between the maximum eigenvalues at various critical points and their respective upper bounds. 
    The eigenvalues are calculated numerically, while the upper bounds are calculated analytically using the main theorem in Eq.~\eqref{eq_main_theorem}.
    All critical points observed in the experiment ($N=3$) described in Appendix~\ref{sec_expe} are plotted.
    }
    \label{fig_eigen_vs_eigenupper}
\end{figure}

\section{The Main Theorem}
\subsection{Taylor Expansion of the Loss Function}
We call $\bm{\theta}=\bm{\theta}^\sharp$ a critical point if it satisfies 
$\pa L / \pa \bm{\theta}|_{\bm{\theta}=\bm{\theta}^\sharp} = \bm{0}_D$.
$\bm{0}_D$ is a $D$-dimensional zero vector.
To analyze the sharpness of the training loss $L$ at a critical point $\bm{\theta}^\sharp$, a second-order Taylor expansion is a useful approach.
The second-order Taylor approximation of $l$ around $\bm{\theta}=\bm{\theta}^\sharp$ is given by
 \ali{
l^{[2]}(\bm{\theta}) = l(\bm{\theta}^\sharp) + \bigg( \frac{\pa l}{\pa \bm{\theta}}\bigg|_{\bm{\theta}=\bm{\theta}^\sharp}\bigg)^\top \ \Delta\bm{\theta}+\frac{1}{2}  \Delta\bm{\theta}^\top \bm{H}_l(\bm{\theta}^\sharp, \bm{\theta}^\sharp) \Delta\bm{\theta} \nn
}
where $\Delta\bm{\theta} = \bm{\theta}-\bm{\theta}^\sharp$.
The second-order Taylor approximation of $L$ around $\bm{\theta}=\bm{\theta}^\sharp$ can be expressed as $L^{[2]}(\bm{\theta}) = \sum_{i=1}^{I} l^{[2]}_i(\bm{\theta})$ due to the linearity of the derivative.
When specifically focusing on the sharpness, it is sufficient to consider only the second-order term. Therefore, the analysis can be restricted to
 \ali{
\bm{H}_{L}(\bm{\theta}^\sharp, \bm{\theta}^\sharp) = \sum_{i=1}^{I} \bm{H}_{l_i}(\bm{\theta}^\sharp, \bm{\theta}^\sharp)  \in \mathbb{R}^{D \times D}. \nn
}
It should be noted that the approximation via Taylor expansion is dominated by higher-order terms when $\| \bm{\theta}-\bm{\theta}^\sharp \|_2 > 1$.
However, since this study utilizes Taylor expansion for analysis around the critical point, we assume $\| \bm{\theta}-\bm{\theta}^\sharp\|_2 \simeq 0$.
Therefore, we note that the Hessian is the dominant factor in evaluating the sharpness of critical points.

\subsection{Upper Bound on the Maximum Eigenvalue of the Hessian}
The maximum eigenvalue of the Hessian of the loss in NNs is a crucial quantity for understanding optimization dynamics \cite{marionDeep2024}.
Therefore, we discuss the upper bound of the maximum eigenvalue in this section.
Let the vector of eigenvalues of $\bm{H}_{L}(\bm{\theta}, \bm{\theta})$ be denoted by
\ali{
\bm{\lambda} = \mat{\lambda_1 & \lambda_2 &\cdots & \lambda_D}^\top, \ 
\lambda_1 \geq \lambda_2 \geq \cdots \geq \lambda_D. \nn
}
That is, $\lambda_1$ and $\lambda_D$ are the maximum and minimum eigenvalues, respectively.
Since $\bm{H}_L(\bm{\theta}, \bm{\theta})$ is a real symmetric matrix, it follows that
\ali{
\forall d \in \mathbb{N}_{\leq D}, \ \lambda_d \in \mathbb{R}. \label{eq_real_eigen}
}
By defining the matrix $\bm{U} = \mat{\bm{u}_1 & \cdots & \bm{u}_D}$ consisting of the eigenvectors $\bm{u}_d$ corresponding to $\lambda_d$, the Hessian can be decomposed as
 \ali{
\bm{H}_{L}(\bm{\theta}, \bm{\theta}) = \bm{U}\mathrm{diag}(\bm{\lambda})\bm{U}^{-1}. \nn
}
The eigenvector $\bm{u}_1$ corresponding to $\lambda_1$ represents the direction in which the second-order derivative of the loss $L$ is maximized.
If $\lambda_D$ is positive, the critical point can be regarded as a local minimum.
In this case, if the maximum eigenvalue $\lambda_1$ is close to zero, it is guaranteed that the critical point has a flat geometry.
Therefore, observing the maximum eigenvalue of the Hessian is essential for understanding the geometry around the critical point of the loss $L$.
However, for $D > 5$, the characteristic equation cannot, in general, be solved analytically, and a closed-form expression for the maximum eigenvalue is therefore not typically obtainable.
To address this, the present study derives a closed-form expression for the upper bound of the maximum eigenvalue.
This is the main theorem of the present study; by applying the Wolkowicz-Styan bound~\cite{wolkowiczBounds1980, merikoskiBounds1997, merikoskiBest2001}, our result can be expressed as follows.
\begin{theorem}[See Eqs.~(1.8), (1.10), (1.11) in \cite{wolkowiczBounds1980}]\label{th_main_theorem}
\ali{
&\lambda_1 \leq \lambda_\mathrm{sup}(\bm{\theta}), \ 
\lambda_\mathrm{sup}(\bm{\theta}) = \mu(\bm{\theta}) + \sqrt{D-1} \sigma(\bm{\theta}), \label{eq_main_theorem} \\
&\mu(\bm{\theta}) = \frac{1}{D}\mathrm{tr} ( \bm{H}_L(\bm{\theta}, \bm{\theta}) ), \label{eq_ave_eigen} \\
&\sigma(\bm{\theta})^2 = \frac{1}{D}\mathrm{tr} ( \bm{H}_L(\bm{\theta}, \bm{\theta})^2 ) 
- \bigg( \frac{1}{D} \mathrm{tr} ( \bm{H}_L(\bm{\theta}, \bm{\theta}) ) \bigg)^2. \label{eq_var_eigen}
}
\end{theorem}
This theorem is equivalent to applying Samuelson's inequality~\cite{tylerLaguerre1999} to the eigenspectrum, and it is known to be tighter than Chebyshev's inequality~\cite{samuelsonHow1968}.
$\mu(\bm{\theta})$ and $\sigma(\bm{\theta})^2$ are the mean and variance of the eigenspectrum, respectively, and their naive definitions are given by
\ali{
\mu(\bm{\theta}) = \frac{1}{D}\sum_{d=1}^D \lambda_d, \ 
\sigma(\bm{\theta})^2 = \frac{1}{D}\sum_{d=1}^D (\lambda_d - \mu(\bm{\theta}))^2. \nn
}
Since closed-form expressions for $\lambda_d$ are unavailable for $D > 5$, simply relying on the naive definitions of the mean and variance fails to yield a useful upper bound in closed form.
However, as shown in Theorem~\ref{th_main_theorem}, the mean and variance of the eigenspectrum can be expressed without explicitly using the individual eigenvalues.
That is, if the closed-form expressions for 
$\mathrm{tr}(\bm{H}_L(\bm{\theta}, \bm{\theta}))$ and 
$\mathrm{tr}(\bm{H}_L(\bm{\theta}, \bm{\theta})^2)$ 
can be derived, a closed-form upper bound for the maximum eigenvalue can be obtained.
This study provides these derivations, and the closed-form expressions for $\mathrm{tr}(\bm{H}_L(\bm{\theta}, \bm{\theta}))$ and $\mathrm{tr}(\bm{H}_L(\bm{\theta}, \bm{\theta})^2)$ are presented in Theorems \ref{theorem_h1_tr} and \ref{theorem_h2_tr}, respectively.

\begin{figure}[t] 
    \centering
    \includegraphics[scale=0.75]{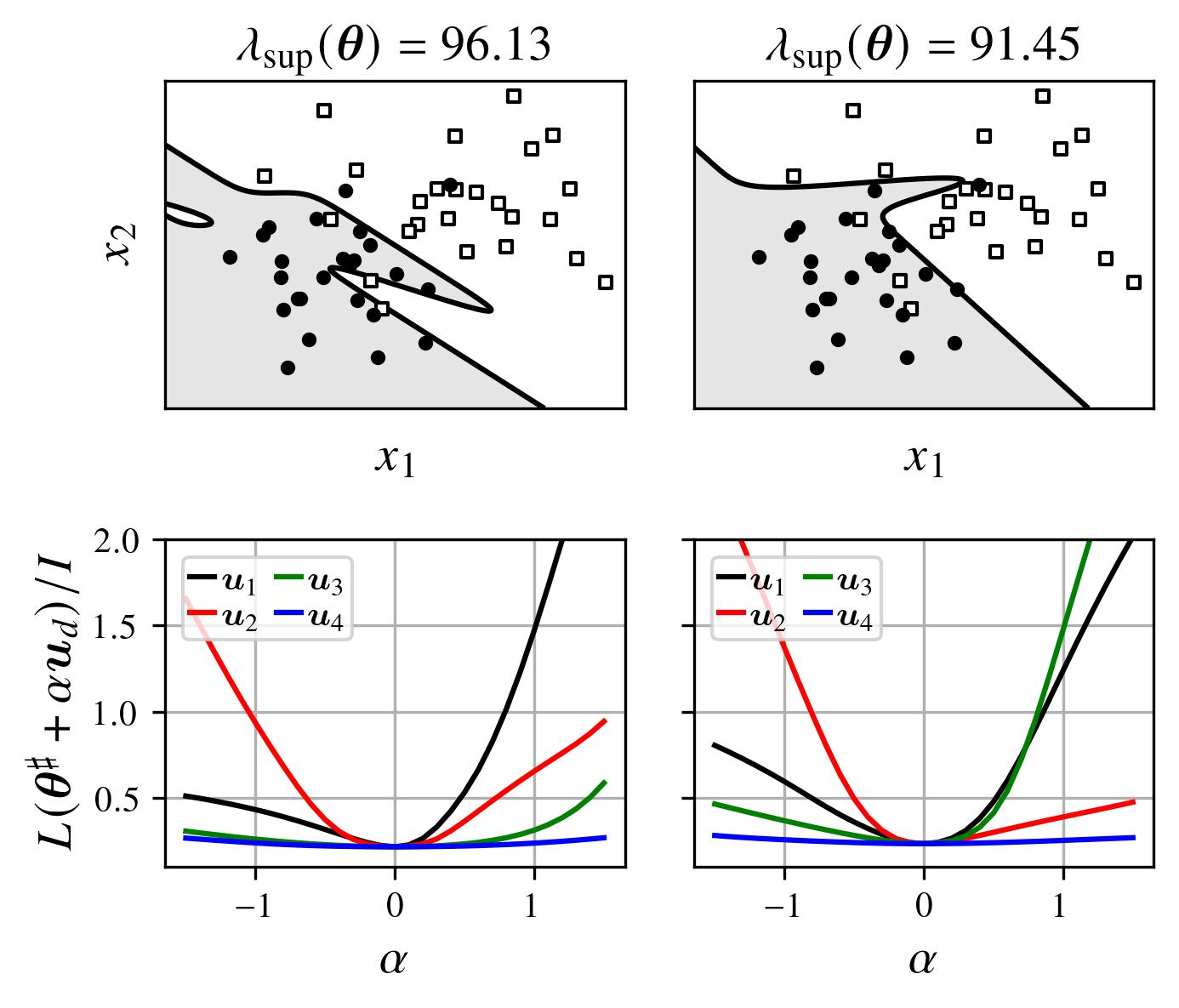}
    \caption{
    (Top) Decision boundaries in the input space for the NNs with the first and second largest upper bounds of the maximum eigenvalue $\lambda_\mathrm{sup}(\bm{\theta})$, observed at all critical points ($N=3$) in the experiments of Appendix \ref{sec_expe}. 
    Data points represent the training data, where white and black indicate Class 0 and Class 1, respectively. 
    (Bottom) One-dimensional visualization of the loss $L$ along the directions of the eigenvectors $\bm{u}_1, \cdots, \bm{u}_4$.
    }
    \label{fig_eigen_max_high}
\end{figure}

The condition for the equality $\lambda_{\mathrm{sup}}(\bm{\theta}) = \lambda_1$ can be summarized as follows.
\begin{proposition}
\ali{
\lambda_1 > 0 \land \Bigg( \bigwedge_{d=2}^D \lambda_d = 0 \Bigg) 
\Rightarrow
\lambda_\mathrm{sup}(\bm{\theta}) = \lambda_1. \label{eq_lambda_equal}
}
\end{proposition}
\begin{proof}
See Appendix~\ref{proof_lambda_equal}.
\end{proof}
This is related to the fact that $\lambda_{\mathrm{sup}}(\bm{\theta})$ approaches $\lambda_1$ as the Hessian becomes more highly degenerate.
It has been pointed out that the Hessian of the loss in NNs exhibits a low-rank structure~\cite{wuDissecting2022}. 
In fact, many studies have reported that the eigenspectrum consists of a small number of large eigenvalues and a vast majority of zero eigenvalues~\cite{yaoPyHessian2020, sagunEigenvalues2017, xiePowerLaw2022, papyanTraces2020}.

\subsection{Relationship between the Upper Bound on the Maximum Eigenvalue and Generalization}
Using the critical points obtained from a set of three-layer NNs, each trained to solve the binary classification problem described in Appendix~\ref{sec_expe}, we investigate several characteristics of the upper bound $\lambda_{\mathrm{sup}}(\bm{\theta})$.
Figure~\ref{fig_eigen_vs_eigenupper} shows a scatter plot of the numerically calculated maximum eigenvalue $\lambda_1$ versus the analytical upper bound $\lambda_{\mathrm{sup}}(\bm{\theta})$ based on Theorem \ref{th_main_theorem}.
From this figure, it can be observed that the proposed upper bound $\lambda_{\mathrm{sup}}(\bm{\theta})$ is close to the maximum eigenvalue $\lambda_1$. Therefore, it can be considered a tight and effective upper bound in practice.
The upper part of Figure~\ref{fig_eigen_max_high} visualizes the decision boundaries for the binary classification problem using the parameters of the critical points where $\lambda_{\mathrm{sup}}(\bm{\theta})$ ranked first and second largest.
The lower part of Figure~\ref{fig_eigen_max_high} visualizes the loss $L(\bm{\theta}^\sharp + \alpha \bm{u})$ along the direction of the eigenvector $\bm{u}$ around the critical points.
To visualize the loss landscape, we employed the method of Goodfellow et al.~\cite{goodfellowQualitatively2015}.
These results confirm that when the upper bound on the maximum eigenvalue is large, the loss at the critical point is sharply peaked, and the decision boundaries are significantly distorted.
In contrast, Figure~\ref{fig_eigen_max_low} presents the results for the cases where $\lambda_{\mathrm{sup}}(\bm{\theta})$ ranks first and second smallest.
This figure demonstrates that when the upper bound on the maximum eigenvalue is small, the loss at the critical point is flat, and the distortion of the decision boundaries is minimal.
Consequently, a classifier with a smaller $\lambda_{\mathrm{sup}}(\bm{\theta})$ can be regarded as a more desirable model for classification.

\begin{figure}[t] 
    \centering
    \includegraphics[scale=0.75]{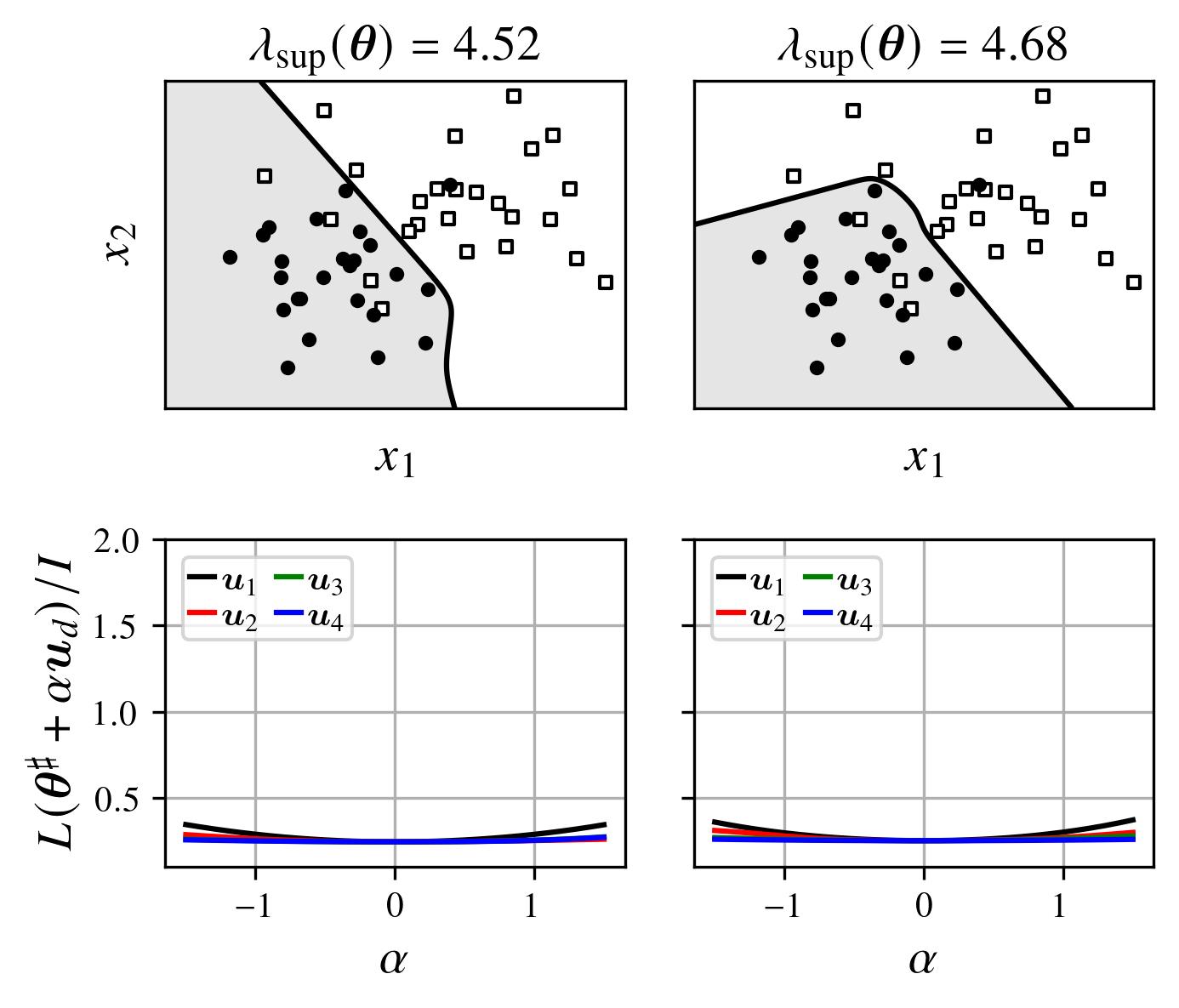}
    \caption{Cases where the upper bound of the maximum eigenvalue, $\lambda_\mathrm{sup}(\bm{\theta})$, ranks first and second smallest. Other descriptions are the same as in Figure~\ref{fig_eigen_max_high}.}
    \label{fig_eigen_max_low}
\end{figure}

\begin{figure}[t] 
    \centering
    \includegraphics[scale=0.65]{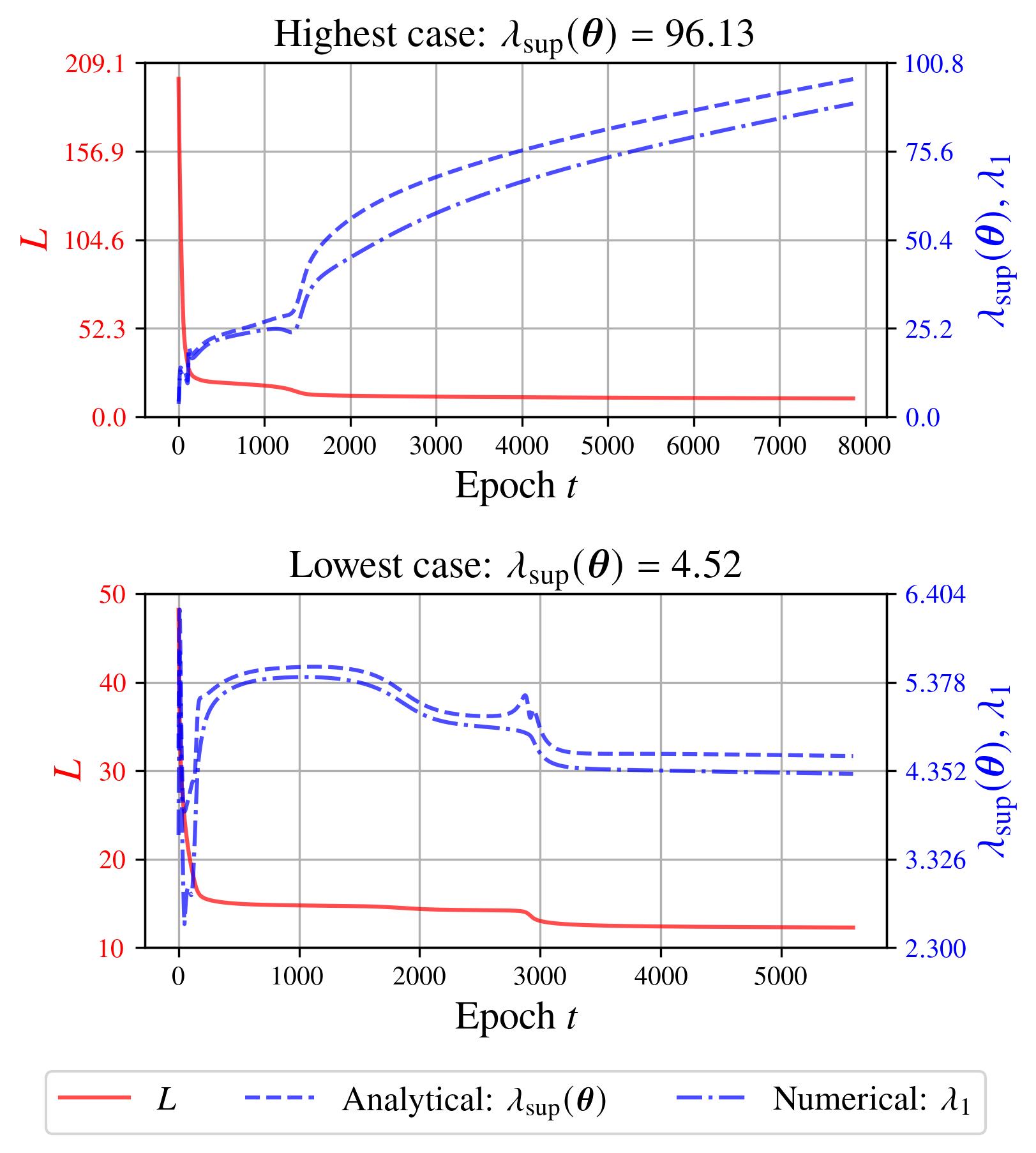}    
    \caption{Dynamics of $L$, $\lambda_{\mathrm{sup}}(\bm{\theta})$, and $\lambda_1$ with respect to training epochs for the highest and lowest $\lambda_{\mathrm{sup}}(\bm{\theta})$ cases.}
    \label{fig_dyn}
\end{figure}

\begin{figure}[t] 
    \centering
    \includegraphics[scale=0.75]{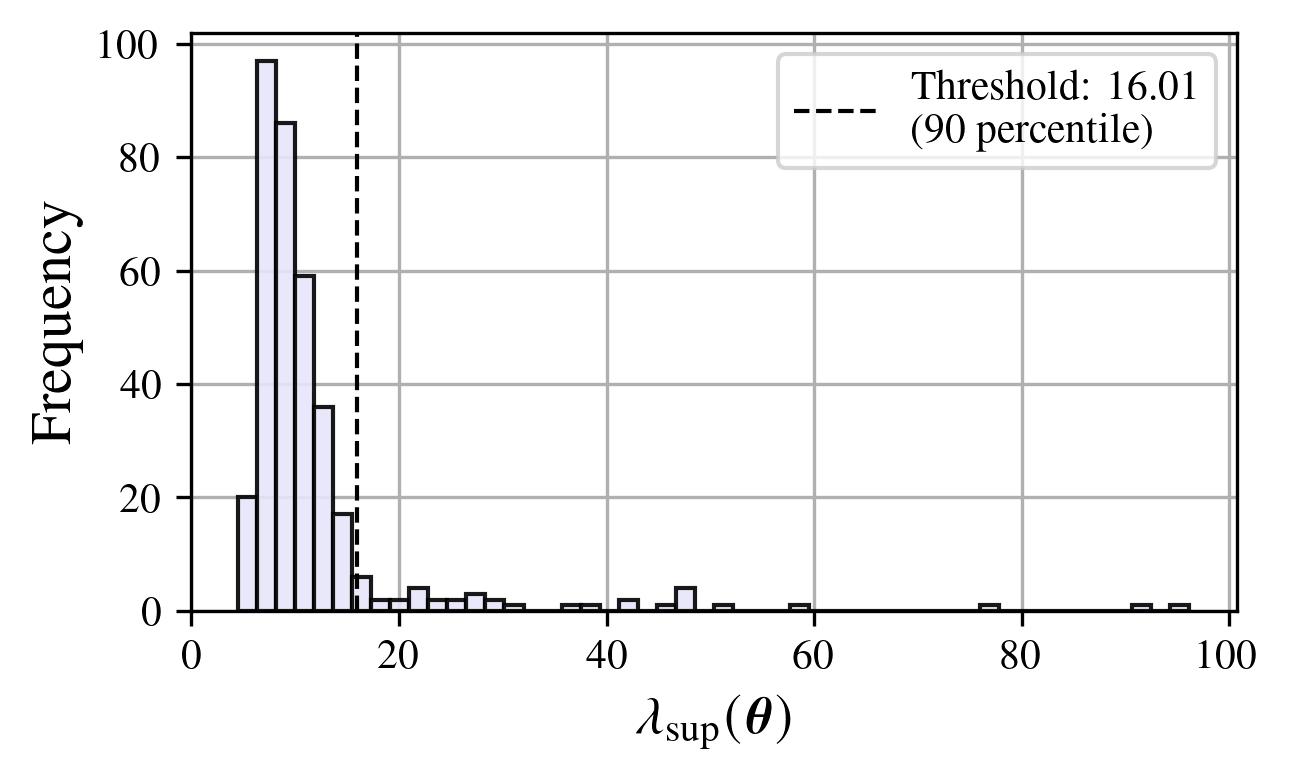}
    \caption{Distribution of the maximum eigenvalue upper bounds at critical points. 
All critical points observed in the experiment ($N=3$) described in Appendix \ref{sec_expe} were used.}
    \label{fig_eigen_hist}
\end{figure}

\begin{figure}[t] 
    \centering
    \includegraphics[scale=0.75]{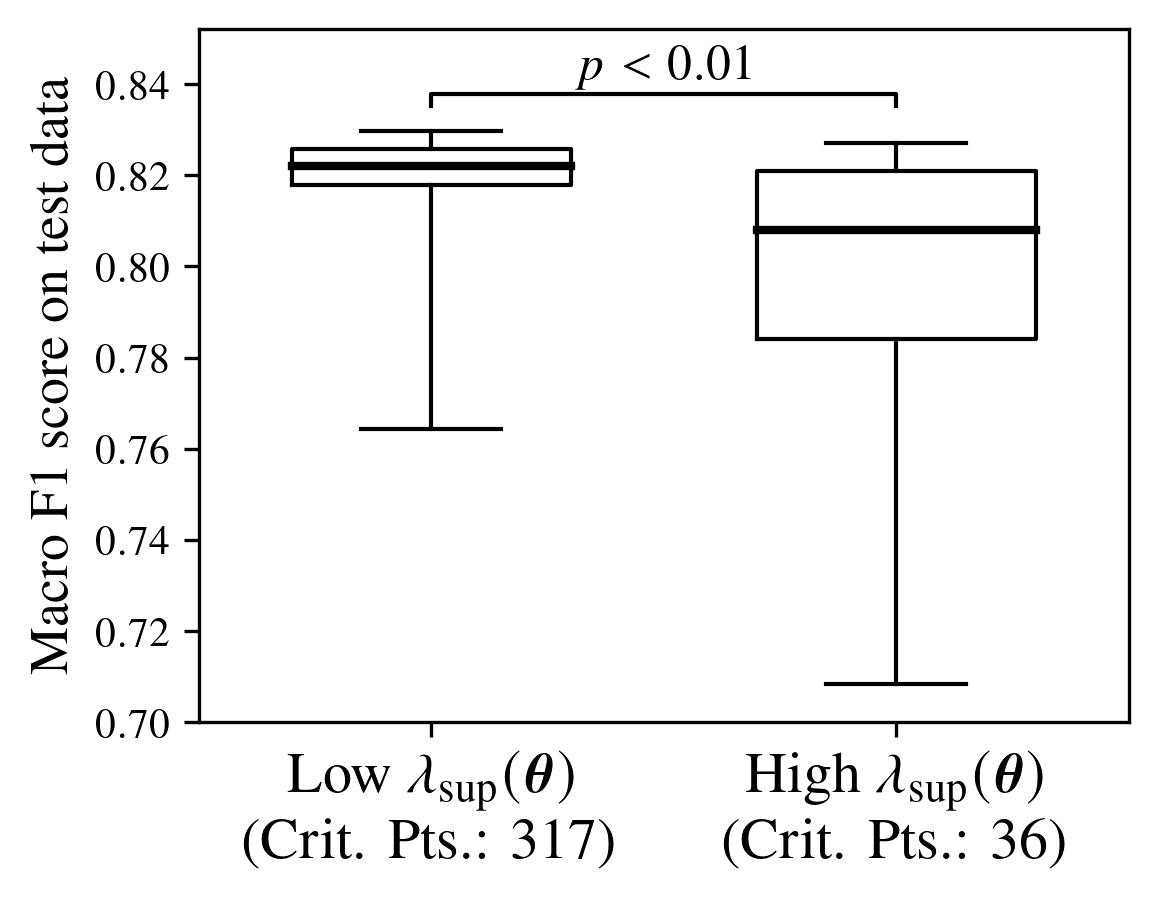}
    \caption{
    Relationship between the maximum eigenvalue upper bound and the Macro F1-score (test data). 
    The boxes represent the 25th--75th percentile range, the whiskers indicate the 10th--90th percentile range, and the thick horizontal lines denote the medians. 
The $p$-value was calculated using the two-sided Mann-Whitney U test. 
Critical points from the experiment ($N=3$) described in Appendix \ref{sec_expe} were used.}
    \label{fig_box_f1}
\end{figure}

Figure~\ref{fig_dyn} shows the dynamics of the loss $L$ and the upper bound $\lambda_{\mathrm{sup}}(\bm{\theta})$ with respect to the training epochs of the gradient descent method.
The upper part shows the case with the highest $\lambda_{\mathrm{sup}}(\bm{\theta})$, while the lower part shows the case with the lowest $\lambda_{\mathrm{sup}}(\bm{\theta})$.
These results reveal two distinct types of learning dynamics: one in which the loss decreases while $\lambda_{\mathrm{sup}}(\bm{\theta})$ increases, and another in which the loss decreases without an increase in $\lambda_{\mathrm{sup}}(\bm{\theta})$.
Furthermore, it can be confirmed that the analytical upper bound $\lambda_{\mathrm{sup}}(\bm{\theta})$ derived in Theorem \ref{th_main_theorem} closely tracks the maximum eigenvalue $\lambda_1$, which can only be determined numerically.

Figure~\ref{fig_eigen_hist} shows the distribution of the upper bound $\lambda_{\mathrm{sup}}(\bm{\theta})$.
This figure shows that the distribution of $\lambda_{\mathrm{sup}}(\bm{\theta})$ is heavily right-skewed; while most critical points are concentrated at low values, there exists a long tail of a few cases with exceptionally high upper bounds.
In other words, these results suggest that while NN training typically converges to flat minima, it can occasionally reach sharp minima.
As shown in Figure~\ref{fig_eigen_max_high}, sharp minima with a large $\lambda_{\mathrm{sup}}(\bm{\theta})$ correspond to inappropriate decision boundaries.
Therefore, we investigated whether a large $\lambda_{\mathrm{sup}}(\bm{\theta})$ leads to a degradation in test performance.
Figure~\ref{fig_box_f1} shows the macro F1-score on the test data.
The ``Low $\lambda_{\mathrm{sup}}(\bm{\theta})$'' group consists of 317 critical points, representing approximately 90\% of the total critical points, which have small maximum eigenvalue upper bounds.
The ``High $\lambda_{\mathrm{sup}}(\bm{\theta})$'' group comprises the remaining 10\% (36 critical points), which represent the outliers with large maximum eigenvalue upper bounds.
The dashed line in the histogram shown in Figure~\ref{fig_eigen_hist} represents this boundary.
Observations of the macro F1-scores reveal that critical points with a low $\lambda_{\mathrm{sup}}(\bm{\theta})$ exhibit stable test performance.
In contrast, those with a high $\lambda_{\mathrm{sup}}(\bm{\theta})$ show significant variance in performance, along with a lower median score.
The results of a two-sided Mann-Whitney U test showed a significant difference at the 1\% level.
This indicates that the upper bound of the maximum eigenvalue is indeed associated with the test performance.
Figure~\ref{fig_eigen_max_high} suggests that the distortion of decision boundaries is the cause of the poor test performance.
Therefore, it is crucial to observe and analyze the upper bound of the maximum eigenvalue.

\begin{figure}[t] 
    \centering
    \includegraphics[scale=0.73]{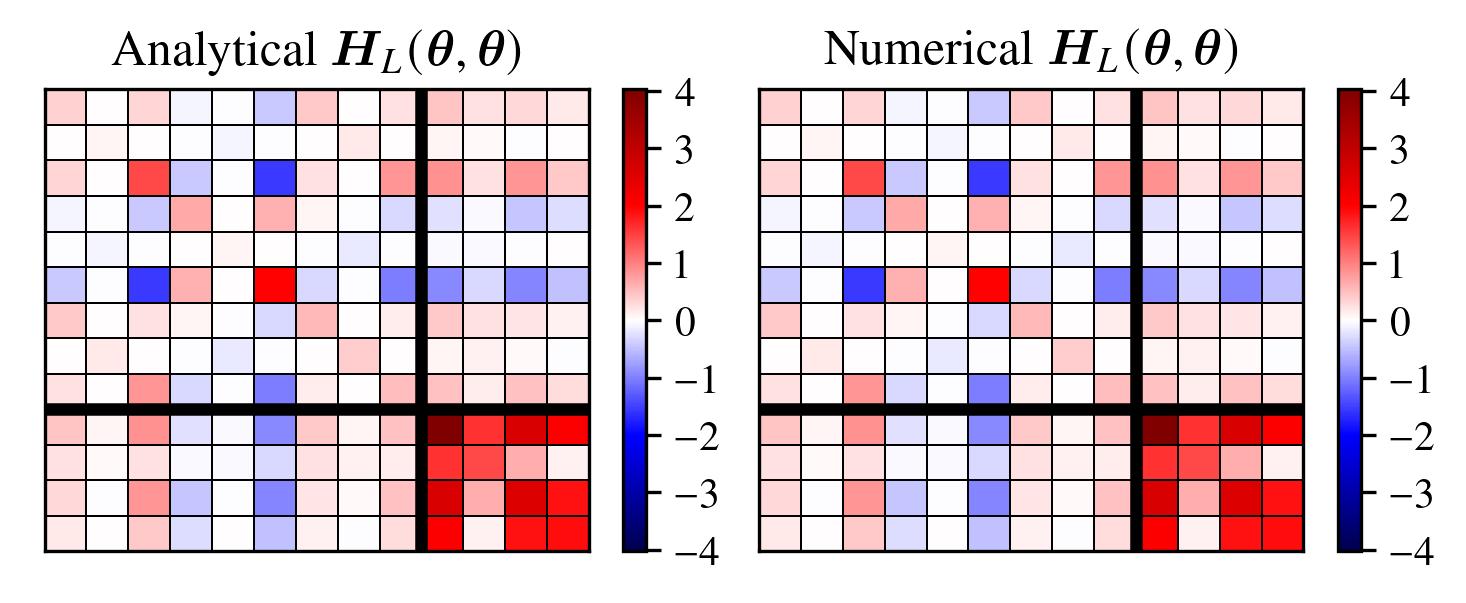}
\caption{Hessian (Matrix size: $13 \times 13$) of a single critical point obtained in the experiment in Appendix \ref{sec_expe}. 
The left shows the analytical solution based on Eqs.~\eqref{eq_ww}--\eqref{eq_wv}, and the right shows the numerical solution obtained by the three-point central difference approximation. 
Thick lines indicate the boundaries of the block components. 
Top-left: $\bm{H}_L(\bm{w}, \bm{w})$, top-right: $\bm{H}_L(\bm{w}, \bm{v})$, bottom-left: $\bm{H}_L(\bm{v}, \bm{w})$, and bottom-right: $\bm{H}_L(\bm{v}, \bm{v})$.}
    \label{fig_hessian_heat}
\end{figure}

\begin{figure}[t] 
    \centering
    \includegraphics[scale=0.73]{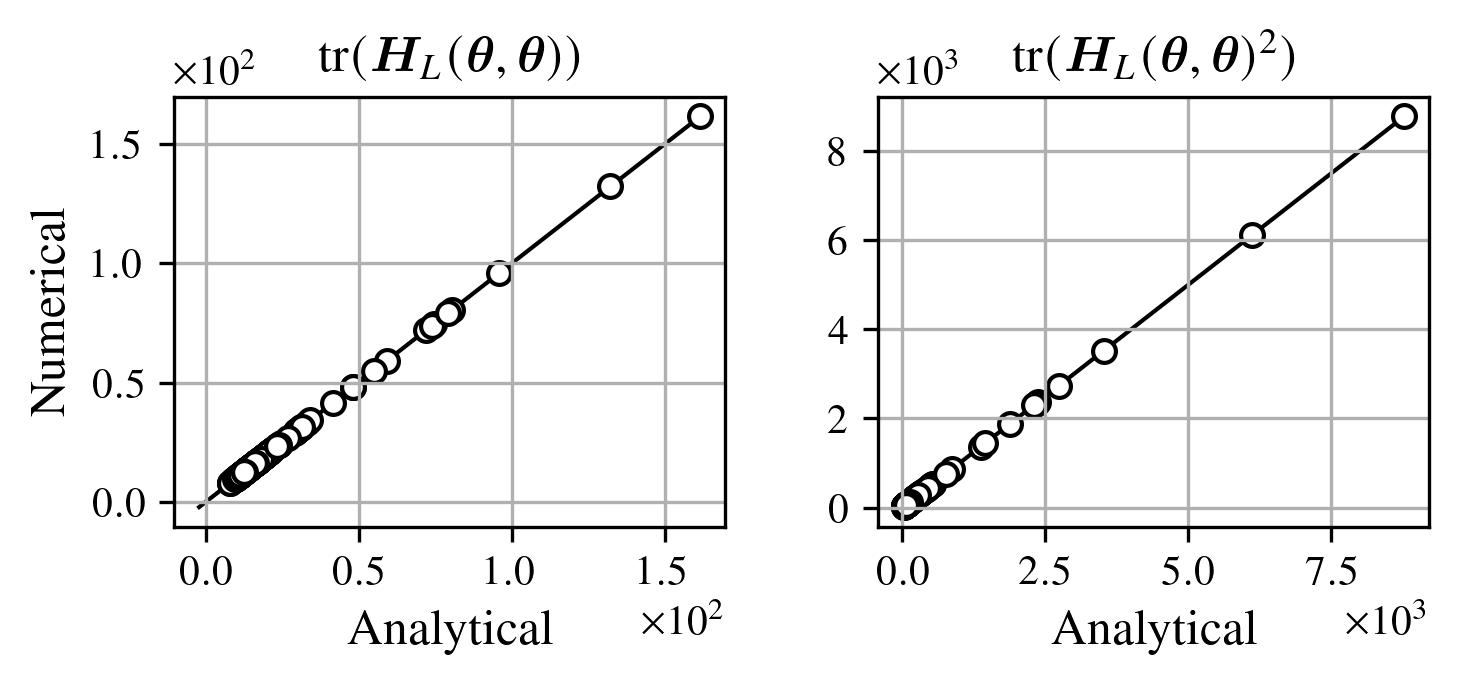}    
    \caption{Comparison of the trace between numerical solutions obtained via ``numpy.trace'' and analytical solutions derived from Theorems \ref{theorem_h1_tr} and \ref{theorem_h2_tr}. All critical points identified in Appendix \ref{sec_expe} are employed.}
    \label{fig_hessian_error}
\end{figure}

\section{Key Ingredients for the Main Theorem}
\subsection{Gradients}
Theorem~\ref{th_main_theorem} includes the Hessian. 
Since the Hessian is a matrix consisting of second-order derivatives, we first obtain the first-order derivatives.
The gradient of the loss $l$ for a single data point shown in Eq.~\eqref{eq_unit_ce} is as follows.
\begin{proposition}
\ali{
 \frac{\pa l}{\pa \bm{\theta}} &= \mat{\frac{\pa l}{\pa \bm{w}} \\ \frac{\pa l}{\pa \bm{v}}} \in \mathbb{R}^{D \times 1}, \label{eq_d_ltheta_1} \\
 \frac{\partial l}{\partial \bm{w}} &= \delta \bm{h}(\bm{x}) \otimes (\bm{F}^\prime(\bm{y})\widetilde{\bm{V}}^\top) \in \mathbb{R}^{(M+1)N \times 1}, \label{eq_d_lW}\\
\frac{\partial l}{\partial \bm{v}} &= \delta\bm{h}(\bm{r})\in \mathbb{R}^{(N+1) \times 1}, \ \delta = p-q. \label{eq_d_lV}
}
\end{proposition}
\begin{proof}
See Appendices \ref{proof_1st_deri_W} and \ref{proof_1st_deri_V}.
\end{proof}
Note that $\otimes$ denotes the Kronecker product.
The gradient of the loss for all data, $\partial L / \partial \bm{\theta}$, is obtained by substituting Eq.~\eqref{eq_d_ltheta_1} into Eq.~\eqref{eq_hessian_L}.

\subsection{Hessian}
The Hessian can be obtained by using the closed-form expression of the gradient.
The Hessian of $l$ is as follows.
\begin{proposition}
\ali{
&\bm{H}_l(\bm{\theta}, \bm{\theta}) = \mat{
\bm{H}_l(\bm{w}, \bm{w}) & \bm{H}_l(\bm{w}, \bm{v}) \\
\bm{H}_l(\bm{v}, \bm{w}) & \bm{H}_l(\bm{v}, \bm{v})
}  \in \mathbb{R}^{D \times D}, \label{eq_H_th2} \\
&\bm{H}_l(\bm{w}, \bm{w}) \in \mathbb{R}^{(M+1)N \times (M+1)N}, 
\bm{H}_l(\bm{v}, \bm{v}) \in \mathbb{R}^{(N+1) \times (N+1)}, \nn \\
&\bm{H}_l(\bm{v}, \bm{w}) \in \mathbb{R}^{(N+1) \times (M+1)N},
\bm{H}_l(\bm{w}, \bm{v}) \in \mathbb{R}^{(M+1)N \times (N+1)}. \nn
}
\end{proposition}
\begin{proof}
See Appendix \ref{sec_hesse_start}.
\end{proof}
\begin{proposition}
\begin{align}
&\bm{H}_l(\bm{w}, \bm{w}) = \bm{h}(\bm{x})\bm{h}(\bm{x})^\top \otimes \nn \\
& \big( s^\prime(z)\bm{F}^\prime(\bm{y})\widetilde{\bm{V}}^\top \widetilde{\bm{V}} \bm{F}^\prime(\bm{y})  +\delta \mathrm{diag}(\widetilde{\bm{V}}^\top) \bm{F}^{\prime\prime}(\bm{y}) \big), \label{eq_ww} \\
&\bm{H}_l(\bm{v}, \bm{v}) =  s^\prime(z) \bm{h}(\bm{r}) \bm{h}(\bm{r})^\top, \label{eq_vv}\\
& \bm{H}_l(\bm{v}, \bm{w}) = \nn \\
& \bm{h}(\bm{x})^\top \otimes \Bigg( s^\prime(z) \bm{h}(\bm{r})  \widetilde{\bm{V}} \bm{F}^\prime(\bm{y}) + \delta \mat{\bm{0}_N^\top \\ \bm{F}^\prime(\bm{y})}\Bigg), \label{eq_vw} \\
&\bm{H}_l(\bm{w}, \bm{v}) = \nn \\
&\bm{h}(\bm{x})  \otimes \big(s^\prime(z) \bm{F}^\prime(\bm{y}) \widetilde{\bm{V}}^\top \bm{h}(\bm{r})^\top  + \delta \mat{\bm{0}_N & \bm{F}^\prime(\bm{y})} \big). \label{eq_wv}
 \end{align}
 \end{proposition}
\begin{proof}
See Appendices \ref{sec_proof_HWW}, \ref{sec_proof_HVV}, and \ref{sec_proof_HVW}.
\end{proof}
Note that the Hessian of the loss for all data, $\bm{H}_L(\bm{\theta}, \bm{\theta})$, is obtained by substituting Eq.~\eqref{eq_H_th2} into Eq.~\eqref{eq_hessian_L}, and is expressed as
\ali{
 \bm{H}_L(\bm{\theta}, \bm{\theta}) 
&= \sum_{i=1}^I \mat{
\bm{H}_{l_i}(\bm{w}, \bm{w}) & \bm{H}_{l_i}(\bm{w}, \bm{v}) \\
\bm{H}_{l_i}(\bm{v}, \bm{w}) & \bm{H}_{l_i}(\bm{v}, \bm{v})
}\nn \\
&= \mat{
\bm{H}_L(\bm{w}, \bm{w}) & \bm{H}_L(\bm{w}, \bm{v}) \\
\bm{H}_L(\bm{v}, \bm{w}) & \bm{H}_L(\bm{v}, \bm{v})
} \in \mathbb{R}^{D \times D}. \nn
}
To evaluate the correctness of the analytical solution, we compared the analytical and numerical solutions using a single critical point obtained from the experiment in Appendix~\ref{sec_expe}.
The results are shown in Figure~\ref{fig_hessian_heat}, and it can be seen that they are consistent with each other.
For this quantitative evaluation, we calculated the Frobenius norm of the matrix $\Delta \bm{H}_L(\bm{\theta}, \bm{\theta})$ obtained by the difference between the numerical and analytical solutions for all critical points in the experiment,
\ali{
\| \Delta \bm{H}_L(\bm{\theta}, \bm{\theta}) \|_\mathrm{F} = \sqrt{\sum_{i=1}^{D}\sum_{j=1}^{D} ((\Delta \bm{H}_L(\bm{\theta}, \bm{\theta}))_{ij})^2}. \nn
}
As a result, the average value of $\| \Delta \bm{H}_L(\bm{\theta}, \bm{\theta}) \|_\mathrm{F}$ for all critical points was $5.44 \times 10^{-5}$.
Since the numerical and analytical Hessians are consistent with each other, it can be concluded that the Hessian derived in this study is correct.

\subsection{Trace of the Hessian}
Since the Hessian is a real symmetric matrix, its eigenvalues are real numbers.
Therefore, the trace of the Hessian, which is a component of Theorem~\ref{th_main_theorem}, is given by
\ali{
\mathrm{tr}(\bm{H}_L(\bm{\theta}, \bm{\theta})) &=  \sum_{d=1}^D \lambda_d \in \mathbb{R}. \nn
}
This can be expressed as follows.
\begin{theorem} \label{theorem_h1_tr}
\ali{
&\mathrm{tr}(\bm{H}_L(\bm{\theta}, \bm{\theta})) 
= \mathrm{tr}(\bm{H}_L(\bm{w}, \bm{w})) + \mathrm{tr}(\bm{H}_L(\bm{v}, \bm{v})), \label{eq_H_l_the2_known_all_vvww} \\
&\mathrm{tr}(\bm{H}_L(\bm{w}, \bm{w})) = 
\big( \bm{s}^\prime \odot (\bm{F}^{\prime\odot 2}_{\mathrm{y}} \widetilde{\bm{V}}^{\odot 2 \top} )  
+\bm{\delta} \odot (\bm{F}^{\prime\prime}_{\mathrm{y}} \widetilde{\bm{V}}^\top ) \big)^\top \nonumber \\
& \cdots (\bm{E}_I + \bm{X}^\top \bm{X} \odot \bm{E}_I) \bm{1}_I, \label{eq_H_l_the2_known_all_ww} \\
&\mathrm{tr}(\bm{H}_L(\bm{v}, \bm{v})) = \bm{s}^{\prime \top}(\bm{E}_I + \bm{R}^\top \bm{R} \odot \bm{E}_I) \bm{1}_I. \label{eq_H_l_the2_known_all_vv}
}
\end{theorem}
\begin{proof}
See Appendix \ref{sec_proof_trace}.
\end{proof}
Here, $\bm{1}_I \in \{1\}^{I \times 1}$ is an $I$-dimensional vector of ones, 
$\bm{E}_I$ is the $I \times I$ identity matrix, 
$\bm{A}^{\odot 2}= \bm{A}\odot \bm{A}$, 
$\bm{X} = \mat{\bm{x}_1\cdots \bm{x}_I } \in \mathbb{R}^{M \times I}$, 
$\bm{R} = \mat{\bm{r}_1\cdots \bm{r}_I } \in \mathbb{R}^{N\times I}$, 
$\bm{s}^\prime = \mat{s^\prime(z_i)} \in \mathbb{R}^{I \times 1}$, 
$\bm{F}^\prime_\mathrm{y} = \mat{\bm{f}^\prime(\bm{y}_1) \cdots \bm{f}^\prime(\bm{y}_I)}^\top \in \mathbb{R}^{I \times N}$, 
$\bm{F}^{\prime\prime}_\mathrm{y} = \mat{\bm{f}^{\prime\prime}(\bm{y}_1) \cdots \bm{f}^{\prime\prime}(\bm{y}_I)}^\top \in \mathbb{R}^{I \times N}$, and $\bm{\delta} = \mat{\delta_i} \in (-1, 1)^{I \times 1}$.
The numerical solutions using the ``numpy.trace'' function and the analytical solutions are shown in the left part of Figure~\ref{fig_hessian_error}, confirming that they are consistent.

Furthermore, the upper bounds of $\mathrm{tr}(\bm{H}_L(\bm{v}, \bm{v}))$ and $\mathrm{tr}(\bm{H}_L(\bm{w}, \bm{w}))$ are as follows.
\begin{theorem}
\ali{
\mathrm{tr}(\bm{H}_{L}(\bm{v}, \bm{v})) < 
\begin{cases}
\frac{1}{4} I (N+1), f: \mathrm{Sigmoid/Tanh} \\
\infty, f: \mathrm{Linear/SmoothReLU/GELU} 
\end{cases}
\label{eq_tr_ub_vv}
}
\end{theorem}
\begin{proof}
See Appendix \ref{sec_proof_tr_ub_vv}.
\end{proof}
\begin{theorem}
\ali{
&\mathrm{tr}(\bm{H}_L(\bm{w}, \bm{w})) \leq \nn \\
& \bigg( \frac{1}{4} \zeta_1 \| \widetilde{\bm{V}} \|^2_\mathrm{F} + \zeta_2 | \widetilde{\bm{V}}\bm{1}_N | \bigg)\bm{1}_I^\top (\bm{E}_I + \bm{X}^\top \bm{X} \odot \bm{E}_I) \bm{1}_I, \nn
}
\ali{
&\zeta_1 = \begin{cases}
1, & f: \mathrm{Linear/Tanh/SmoothReLU}\\
\frac{1}{16}, & f: \mathrm{Sigmoid}\\
1.272..., & f: \mathrm{GELU}
\end{cases}, \nn \\
&\zeta_2 = \begin{cases}
0, & f: \mathrm{Linear}\\
\frac{\sqrt{3}}{18} = 0.096..., & f: \mathrm{Sigmoid}\\
\frac{4\sqrt{3}}{9} = 0.769..., & f: \mathrm{Tanh}\\
\frac{1}{4}, & f: \mathrm{SmoothReLU} \\
\sqrt{\frac{2}{\pi}} = 0.797..., & f: \mathrm{GELU}
\end{cases}.\label{eq_tr_ub_ww}
}
\end{theorem}
\begin{proof}
See Appendix \ref{sec_proof_tr_ub_ww}.
\end{proof}
Furthermore, if the inputs are normalized, the following holds.
\begin{proposition}
\ali{
\max_{\bm{X} \in [0, 1]^{M \times I}} \bm{1}_I^\top (\bm{E}_I + \bm{X}^\top \bm{X} \odot \bm{E}_I) \bm{1}_I 
= I(1+M). 
\label{eq_max_largex}
}
\end{proposition}
\begin{proof}
See Appendix \ref{proof_max_largex}.
\end{proof}
When the maximum of each component of the input $\bm{x}$ is normalized to 1, the upper bound of $\mathrm{tr}(\bm{H}_L(\bm{w}, \bm{w}))$ takes a simpler form, and it can be confirmed that the number of data $I$ and the input dimension $M$ increase the upper bound.

\subsection{Trace of the Square Hessian}
$\mathrm{tr} ( \bm{H}_L(\bm{\theta}, \bm{\theta})^2 )$, which is a component of Theorem~\ref{th_main_theorem}, is equal to the sum of the squares of the eigenvalues.
Furthermore, since the eigenvalues of a real symmetric matrix are real numbers,
\ali{
\mathrm{tr}(\bm{H}_L(\bm{\theta}, \bm{\theta})^2) &=  \sum_{d=1}^D (\lambda_d)^2  \geq 0 \nn
}
holds.
This can be expressed as follows.
\begin{theorem} \label{theorem_h2_tr}
\ali{
&\mathrm{tr} ( \bm{H}_L(\bm{\theta}, \bm{\theta})^2 )
= \langle \bm{\Phi}, (\bm{J}_I + \bm{X}^\top \bm{X})^{\odot 2} \rangle_\mathrm{F} \nn \\
&+ 2\langle \bm{\Psi}, (\bm{J}_I + \bm{X}^\top \bm{X}) \rangle_\mathrm{F}
+ \langle \bm{\Omega}, (\bm{J}_I + \bm{R}^\top \bm{R})^{\odot 2} \rangle_\mathrm{F}. \label{eq_h2_tr_fin}
}
\end{theorem}
\begin{proof}
See Appendices \ref{sec_proof_h2_prop_step1}, \ref{sec_proof_h2_prop_step2}, \ref{sec_proof_h2_prop_step3}, \ref{sec_proof_h2_prop_step4}, and \ref{sec_proof_h2_prop_step5}.
\end{proof}
Here, $\langle \cdot, \cdot \rangle_\mathrm{F}$ is the Frobenius inner product, 
$\bm{J}_I = \bm{1}_I \bm{1}_I^\top$, 
$\bm{\Phi} = \mat{\phi_{ij}}_{i, j \in \mathbb{N}_{\leq I}}$, 
$\bm{\Psi} = \mat{\psi_{ij}}_{i, j \in \mathbb{N}_{\leq I}}$, and 
$\bm{\Omega} = \mat{\omega_{ij}}_{i, j \in \mathbb{N}_{\leq I}}$.
$\phi_{ij}$, $\psi_{ij}$, $\omega_{ij}$ are bilinear forms, given by
\ali{
\phi_{ij} &= \bm{o}_{i}^\top \bm{\Phi}_{ij} \bm{o}_{j} \in \mathbb{R}, \label{eq_phi_ij}\\
\psi_{ij} &= \bm{o}_{i}^\top \bm{\Psi}_{ij} \bm{o}_{j} \in \mathbb{R}, \label{eq_psi_ij}\\
\omega_{ij} &= \bm{o}_{i}^\top \bm{\Omega}_{ij} \bm{o}_{j} \in (0, 1/16]. \label{eq_ome_ij}
}
In addition,
\ali{
\bm{o}_i =
\mat{
s^\prime(z_i)\\
\delta_i
}, \ \bm{\Omega}_{ij} = \mat{1 & 0\\0 & 0}. \nn
}
Furthermore, 
\ali{
(\bm{\Phi}_{ij})_{11} &= 
(\widetilde{\bm{V}} \bm{F}^\prime(\bm{y}_j) \bm{F}^\prime(\bm{y}_i)\widetilde{\bm{V}}^\top)^2, \nn \\
(\bm{\Phi}_{ij})_{12} &= \widetilde{\bm{V}}  \bm{F}^\prime(\bm{y}_i) \mathrm{diag}(\widetilde{\bm{V}}^\top) \bm{F}^{\prime\prime}(\bm{y}_j) \bm{F}^\prime(\bm{y}_i)\widetilde{\bm{V}}^\top, \nn \\
(\bm{\Phi}_{ij})_{21} &=  \widetilde{\bm{V}} \bm{F}^\prime(\bm{y}_j) \mathrm{diag}(\widetilde{\bm{V}}^\top)\bm{F}^{\prime\prime}(\bm{y}_i) \bm{F}^\prime(\bm{y}_j)\widetilde{\bm{V}}^\top, \nn \\
 (\bm{\Phi}_{ij})_{22} &= \widetilde{\bm{V}} \bm{F}^{\prime\prime}(\bm{y}_i) \bm{F}^{\prime\prime}(\bm{y}_j) \widetilde{\bm{V}}^\top, \nn \\
(\bm{\Psi}_{ij})_{11} &= (1+\bm{f}(\bm{y}_i)^\top \bm{f}(\bm{y}_j)) \widetilde{\bm{V}} \bm{F}^\prime(\bm{y}_j) \bm{F}^\prime(\bm{y}_i) \widetilde{\bm{V}}^\top, \nn \\
(\bm{\Psi}_{ij})_{12} &= \bm{f}(\bm{y}_i)^\top \bm{F}^\prime(\bm{y}_j) \bm{F}^\prime(\bm{y}_i) \widetilde{\bm{V}}^\top, \nn \\
(\bm{\Psi}_{ij})_{21} &= \widetilde{\bm{V}} \bm{F}^\prime(\bm{y}_j) \bm{F}^\prime(\bm{y}_i) \bm{f}(\bm{y}_j), \nn \\
 (\bm{\Psi}_{ij})_{22} &= \bm{f}^\prime(\bm{y}_i)^\top \bm{f}^\prime(\bm{y}_j). \nn
}
Note that the trace of the square Hessian in summation form is provided in Eq.~\eqref{eq_h2_tr_fin_scc}.
The numerical solutions obtained by the ``numpy.trace'' function and the analytical solutions are shown on the right of Figure~\ref{fig_hessian_error}, and it has been confirmed that they match.

The upper bound of the trace of the square Hessian is as follows.
\begin{theorem}
\ali{
&0 \leq \mathrm{tr} ( \bm{H}_L(\bm{\theta}, \bm{\theta})^2 ) \leq \| \bm{\Psi} \|_\mathrm{F}^2 + \nn \\
& ( 1 + \phi_\mathrm{max} ) \| \bm{J}_I + \bm{X}^\top \bm{X} \|_\mathrm{F}^2
+ \frac{1}{16} \| \bm{J}_I + \bm{R}^\top \bm{R} \|_\mathrm{F}^2, \label{eq_h2tr_upper_mat}
}
where
\ali{
\phi_\mathrm{max} = \max_{i, j \in \mathbb{N}_{\leq I}}\phi_{ij} \geq 0. \label{eq_phi_max}
}
\end{theorem}
\begin{proof}
See Appendix~\ref{sec_proof_2htr_upper_v2}.
\end{proof}
From this, it can be interpreted that the inner products of the data in the input and hidden layers, $\bm{x}^\top_i \bm{x}_j$ and $\bm{r}^\top_i \bm{r}_j$, contribute to the upper bound of the sharpness.
Furthermore, the upper bound of $\mathrm{tr} ( \bm{H}_L(\bm{\theta}, \bm{\theta})^2 )$ can be simplified as follows.
\begin{proposition}
\ali{
\max_{\bm{X} \in [0, 1]^{M \times I}} \| \bm{J}_I + \bm{X}^\top \bm{X} \|_\mathrm{F}^2 =  I^2 (1+M)^2.\label{eq_jxx_cond1}
}
\end{proposition}
\begin{proof}
See Appendix~\ref{pr_jxx_cond1}.
\end{proof}
This is the case where the maximum value of each component of the input $\bm{x}$ is normalized to 1.
From this, it can be said that the training data size $I$ and the input dimensionality $M$ increase the upper bound.
Furthermore, when the activation function is a saturation function such as Sigmoid or Tanh, the following holds.
\begin{proposition}
\ali{
&f: \mathrm{Sigmoid} \lor \mathrm{Tanh} \Rightarrow  \nn \\
&\sup_{\bm{R}} \| \bm{J}_I + \bm{R}^\top \bm{R} \|_\mathrm{F}^2 =  I^2 (1+N)^2. \label{eq_jrr_cond1}
}
\end{proposition}
\begin{proof}
See Appendix~\ref{pr_jxx_cond1}.
\end{proof}
In other words, as the dimensionality of the hidden layer $N$ increases, the upper bound of the trace of the square Hessian increases.

\subsection{Asymptotic behavior}
Here, we consider the case where the upper bound of the maximum eigenvalue converges to zero.
\begin{theorem}
\ali{
\bm{\delta} \rightarrow \bm{0}_I \Rightarrow \lambda_\mathrm{sup}(\bm{\theta}) \rightarrow 0. \label{eq_limit_u_eigenmax}
}
\end{theorem}
\begin{proof}
See Appendix~\ref{proof_limit_u_eigenmax}.
\end{proof}
The condition $\bm{\delta} \rightarrow \bm{0}_I$ implies $L \rightarrow 0$; thus, it represents a state where the NN has correctly predicted all training data with maximum confidence.
In general, forcing a strong fit to all training data leads to overfitting, which in turn results in a degradation of generalization performance.
However, the proposition shown in Eq.~\eqref{eq_limit_u_eigenmax} asserts that when the model is overfitted to the limit, $\lambda_\mathrm{sup}(\bm{\theta})$ becomes zero, resulting in a flat solution.
Therefore, in extreme scenarios, there may be cases where generalization cannot be fully explained by sharpness alone, and caution is required.

\begin{figure}[t] 
    \centering
    \includegraphics[scale=0.66]{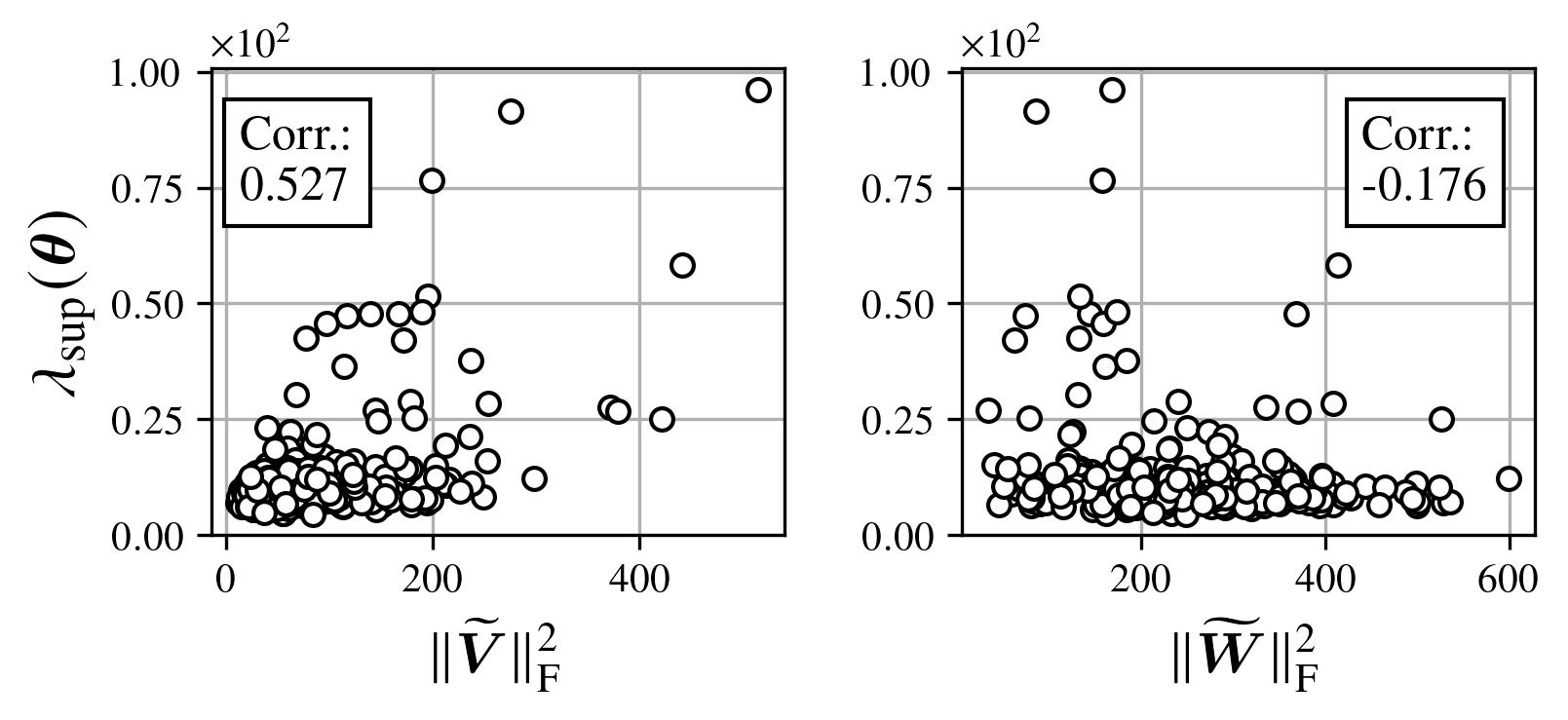}
    \caption{Relationship between the Frobenius norm of the parameters and the upper bound of the maximum eigenvalue.
    A total of 353 critical points mentioned in Appendix \ref{sec_expe} were utilized ($N=3$).}
    \label{fig_norm_vs_upper}
\end{figure}

\section{Discussion}
From Theorem~\ref{th_main_theorem}, it can be interpreted that a decrease in the squared trace $\mathrm{tr} ( \bm{H}_L(\bm{\theta}, \bm{\theta})^2 ) $ leads to a reduction in the upper bound of the maximum eigenvalue $\lambda_{\mathrm{sup}}(\bm{\theta})$.
On the other hand, it can be observed that the trace $\mathrm{tr} ( \bm{H}_L(\bm{\theta}, \bm{\theta}) )$ has both positive and negative impacts on the upper bound $\lambda_{\mathrm{sup}}(\bm{\theta})$.
To simplify the interpretation, we assume that the critical points reached during training possess an ideal shape; specifically, the Hessian is assumed to be a positive semi-definite matrix.
In this case, the following relationship holds between $\mathrm{tr} ( \bm{H}_L(\bm{\theta}, \bm{\theta}) )$ and $\mathrm{tr} ( \bm{H}_L(\bm{\theta}, \bm{\theta})^2 )$.
\begin{proposition}\label{prop_tr_th2}
 \ali{
&\bm{H}_L(\bm{\theta}, \bm{\theta}) \succeq 0 \Rightarrow \nn \\
&\frac{1}{D}\mathrm{tr} ( \bm{H}_L(\bm{\theta}, \bm{\theta}) )^2 
\leq \mathrm{tr} ( \bm{H}_L(\bm{\theta}, \bm{\theta})^2 ) 
\leq \mathrm{tr} ( \bm{H}_L(\bm{\theta}, \bm{\theta}) )^2. \label{eq_btw_tr_tr2}
 }
\end{proposition}
\begin{proof}
See Appendix \ref{sec_proof_btw_tr_tr2}.
\end{proof}
This proposition asserts that increasing $\mathrm{tr} ( \bm{H}_L(\bm{\theta}, \bm{\theta}))$ results in an elevation of both the lower and upper bounds of $\mathrm{tr} ( \bm{H}_L(\bm{\theta}, \bm{\theta})^2 )$.
According to Theorem~\ref{th_main_theorem}, an increase in the squared trace $\mathrm{tr} ( \bm{H}_L(\bm{\theta}, \bm{\theta})^2 )$ leads to an increase in the upper bound of the maximum eigenvalue $\lambda_{\mathrm{sup}}(\bm{\theta})$; therefore, it is also crucial to reduce the trace $\mathrm{tr} ( \bm{H}_L(\bm{\theta}, \bm{\theta}))$.

\begin{figure}[t] 
    \centering
    \includegraphics[scale=0.75]{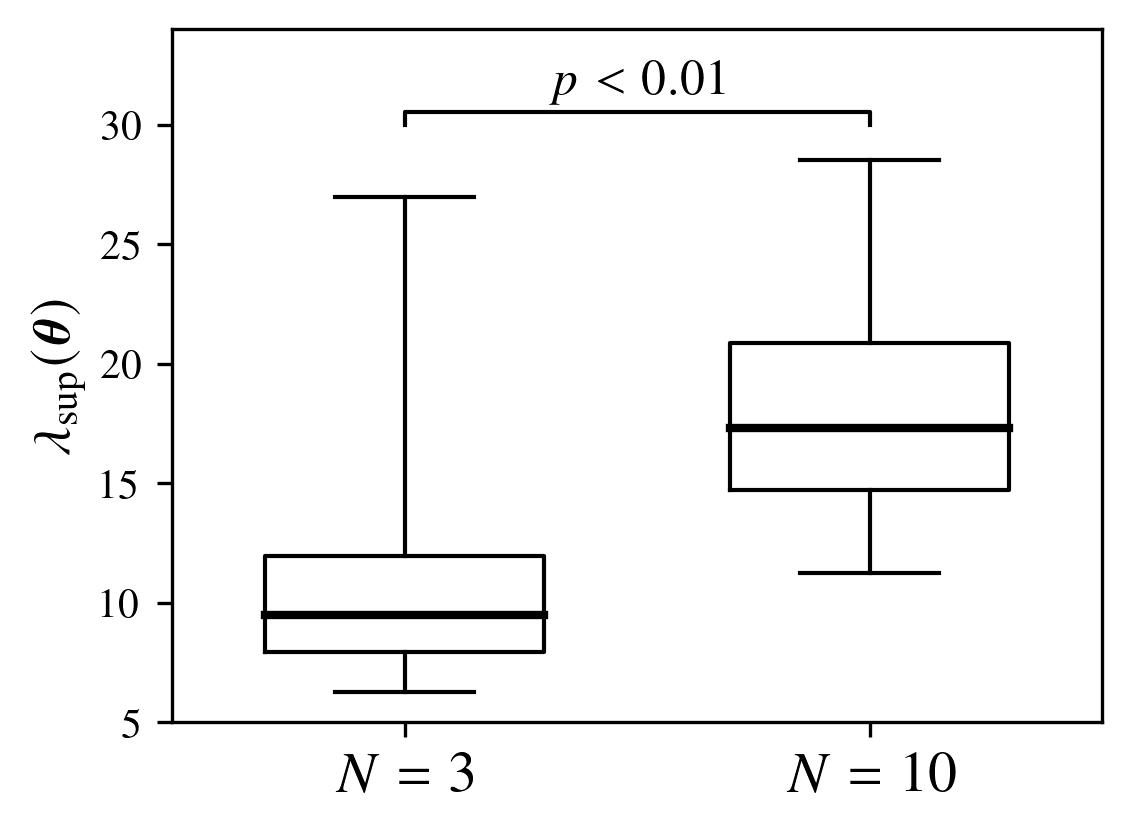}
    \caption{Relationship between the upper bound of the maximum eigenvalue and the dimensionality of the hidden layer $N$. 
    The boxes represent the 25--75\% range, the whiskers indicate the 10--90\% range, and the thick horizontal lines denote the median. 
    The $p$-value is obtained from the two-sided Mann-Whitney U test.}
    \label{fig_box_upper}
\end{figure}

\begin{figure}[t] 
    \centering
    \includegraphics[scale=0.67]{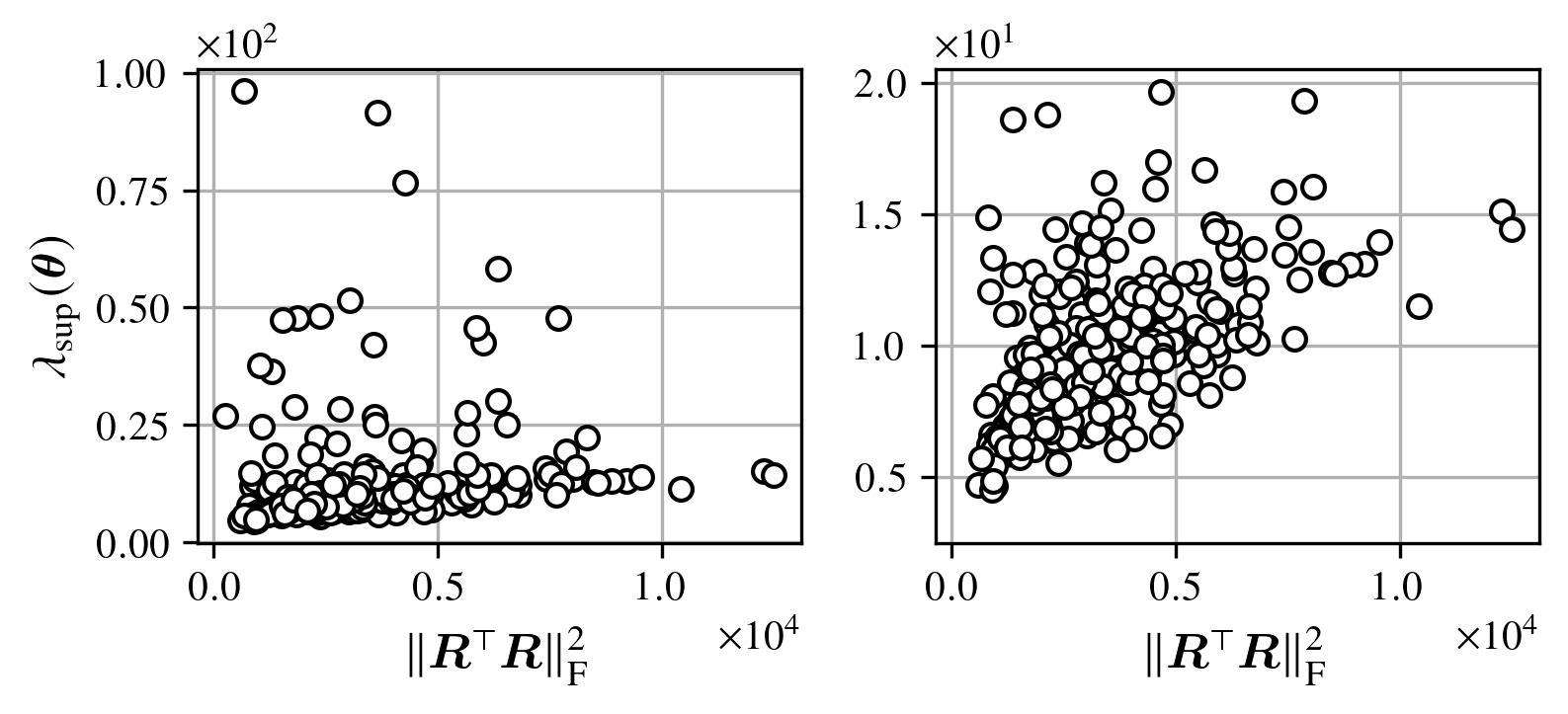}
    \caption{Relationship between the Frobenius norm of the hidden layer and the upper bound of the maximum eigenvalue. 
    The left part shows a scatter plot of all critical points, while the right part provides an enlarged view of the same plot. The critical points used were obtained as described in Appendix~\ref{sec_expe} ($N=3$).}
    \label{fig_RF_vs_upper}
\end{figure}

The form of the upper bound of the trace in Eq.~\eqref{eq_tr_ub_ww} indicates that the norm of the affine mapping parameters from the hidden layer to the output layer, $\|\widetilde{\bm{V}}\|^2_\mathrm{F}$, contributes to an increase in the upper bound of the maximum eigenvalue, $\lambda_{\mathrm{sup}}(\bm{\theta})$.
To investigate whether actual observations support this theoretical implication, we plotted a scatter diagram of $\|\widetilde{\bm{V}}\|^2_\mathrm{F}$ versus $\lambda_{\mathrm{sup}}(\bm{\theta})$, utilizing all critical points observed in the experiments described in Appendix~\ref{sec_expe}.
This result is shown in the left part of Figure~\ref{fig_norm_vs_upper}, and it can be seen that $\lambda_\mathrm{sup}(\bm{\theta})$ tends to increase as $\|\widetilde{\bm{V}}\|^2_\mathrm{F}$ becomes larger.
For reference, we also created a scatter plot of $\|\widetilde{\bm{W}}\|^2_\mathrm{F}$ versus $\lambda_\mathrm{sup}(\bm{\theta})$.
As shown in the right part of Figure~\ref{fig_norm_vs_upper}, no clear relationship was observed between the parameters used for the transformation from the input to the hidden layer and $\lambda_\mathrm{sup}(\bm{\theta})$.
Therefore, it can be concluded that suppressing the scale of the parameters from the hidden layer to the output layer, such as by applying L2 regularization, is crucial.
Furthermore, Eq.~\eqref{eq_tr_ub_vv} reveals that the dimensionality of the hidden layer, $N$, contributes to an increase in the upper bound of the trace.
Therefore, we compared the upper bounds of the maximum eigenvalue at the critical points for $N=3$ and $N=10$.
This result is shown in Figure~\ref{fig_box_upper}. 
A comparison of the upper bounds between the two groups using a two-sided Mann-Whitney U test revealed a significant difference at the 1\% level.
Consequently, it can be concluded that there is a significant relationship between the dimensionality of the hidden layer, $N$, and $\lambda_\mathrm{sup}(\bm{\theta})$.

Next, we examine Eq.~\eqref{eq_h2tr_upper_mat}.
In this equation, we can identify the presence of the Frobenius norms of the inner products of the data matrices, namely $\|\bm{X}^\top\bm{X}\|_\mathrm{F}^2$ and $\|\bm{R}^\top\bm{R}\|_\mathrm{F}^2$.
Therefore, it can be said that the inner products of the input data, $\bm{x}_i^\top \bm{x}_j$, and those of the hidden layer data, $\bm{r}_i^\top \bm{r}_j$, are associated with the upper bound of $\bm{H}_L(\bm{\theta}, \bm{\theta})^2$.
In other words, a lower degree of orthogonality among the training data samples is expected to increase the upper bound of the maximum eigenvalue $\lambda_\mathrm{sup}(\bm{\theta})$. 
Since the terms involve the squared inner products, this implies that any strong alignment, regardless of whether it is in the same or opposite direction, contributes to a larger upper bound.
However, it should be noted that when the input data $\bm{x}$ changes, the loss function $L$ being analyzed also changes. 
Therefore, it is difficult to experimentally verify whether the inner products of the input data affect the upper bound of the maximum eigenvalue in isolation.
Accordingly, we used the critical points obtained in Appendix~\ref{sec_expe} to verify whether a lower degree of orthogonality among the hidden layer data leads to an increase in $\lambda_\mathrm{sup}(\bm{\theta})$.
These results are shown in Figure~\ref{fig_RF_vs_upper}, where the left part shows a scatter plot of all critical points and the right part provides an enlarged view of the lower region.
From this figure, it can be observed that as $\|\bm{R}^\top\bm{R}\|_\mathrm{F}^2$ increases, the lower bound of $\lambda_\mathrm{sup}(\bm{\theta})$ also rises.
Therefore, to reduce the sharpness of critical points, it is important to increase the degree of orthogonality among the data samples.

\section{Conclusion}
Previous studies on the loss landscapes of NNs have primarily focused on analyses based on numerical approximations or theoretical investigations restricted to linear networks.
In other words, there has been a lack of research providing a closed-form expression for sharpness in nonlinear and smooth multilayer NNs, which are widely used in practical deep learning.
Therefore, in this study, we derived a closed-form expression for the upper bound of the maximum eigenvalue of the Hessian, specifically targeting the cross-entropy loss of nonlinear and smooth multilayer NNs.
This explicitly clarifies the influence of model parameters and training data on sharpness. 
These results are expected to contribute to the theoretical advancement of deep learning.
Although this analysis is limited to three-layer models, we intend to extend it to deeper architectures in future work, aiming to further advance deep learning theory.

\appendices

\begin{figure*}[t] 
    \centering
    \includegraphics[scale=0.77]{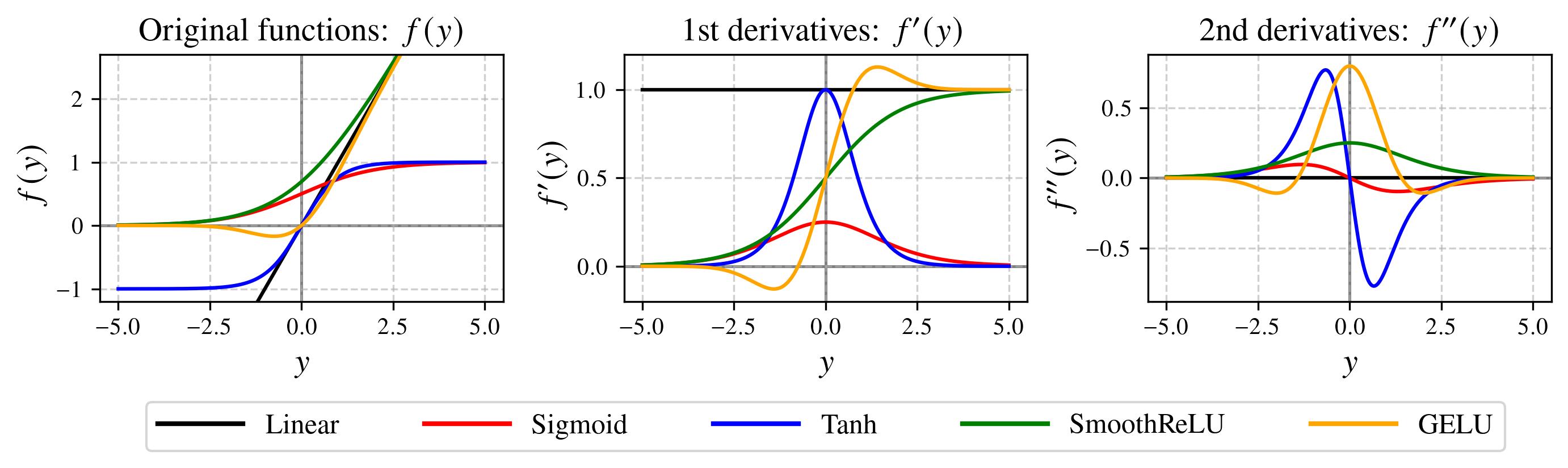}
    \caption{Activation function $f(y)$ and its first and second derivatives.}
    \label{fig_acts5}
\end{figure*}

\section{Activation Functions}
\subsection{Derivatives and Matrix Representations of Activation Functions} \label{sec_activation_base}
Let $f(y) \in \mathbb{R}$ be an activation function for a scalar value $y \in \mathbb{R}$.
We denote its derivatives as
\ali{
f^{\prime}(y) = \frac{\md f(y)}{\md y}, \ 
f^{\prime\prime}(y) = \frac{\md f^\prime(y)}{\md y}, \ 
f^{\prime\prime\prime}(y) = \frac{\md f^{\prime\prime}(y)}{\md y}. \nn
}
We define the vector-valued activation function and its derivatives for a vector $\bm{y} \in \mathbb{R}^{N\times 1}$ as
$\bm{f}(\bm{y}) = \mat{f(y_n)} \in \mathbb{R}^{N\times 1}$, 
$\bm{f}^{\prime}(\bm{y}) = \mat{f^{\prime}(y_n)} \in \mathbb{R}^{N\times 1}$, 
$\bm{f}^{\prime\prime}(\bm{y}) = \mat{f^{\prime\prime}(y_n)} \in \mathbb{R}^{N\times 1}$.

The Jacobian matrices of $\bm{f}(\bm{y})$ and $\bm{f}^\prime(\bm{y})$ with respect to the input $\bm{y}$ are given by
\ali{
\frac{\partial \bm{f}(\bm{y})}{\partial \bm{y}} 
&= \frac{\pa }{\pa \bm{y}}\bm{f}(\bm{y})^\top
=\mathrm{diag}(\bm{f}^\prime(\bm{y}) ), \nn \\
\frac{\partial \bm{f}^\prime(\bm{y})}{\partial \bm{y}} 
&= \frac{\pa }{\pa \bm{y}}\bm{f}^\prime(\bm{y})^\top
=\mathrm{diag}(\bm{f}^{\prime\prime}(\bm{y}) ). \nn
}
To represent this concisely, we define
\ali{
\bm{F}^{\prime}(\bm{y}) = \mathrm{diag}(\bm{f}^{\prime}(\bm{y})), \ 
\bm{F}^{\prime \prime}(\bm{y}) = \mathrm{diag}(\bm{f}^{\prime \prime}(\bm{y})). \label{eq_jacob_sigm}
}
Note that the five activation functions considered in this study are shown in Figure~\ref{fig_acts5}.

\subsection{Linear activation}\label{sec_lin_details}
The linear activation and its first and second-order derivatives are given by
\ali{
f(y) = y \in \mathbb{R}, \ f^\prime(y) = 1, \ f^{\prime\prime}(y) = 0. \label{eq_s_lin_max}
}
Furthermore,
\ali{
\max_{y \in \mathbb{R} } f^\prime(y) &= 
\min_{y \in \mathbb{R} } f^\prime(y) = 1, \label{eq_s_lin_1st_max} \\
\max_{y \in \mathbb{R} } f^{\prime\prime}(y) &= 
\min_{y \in \mathbb{R} } f^{\prime\prime}(y) = 0. \label{eq_s_lin_2nd_max}
}

\subsection{Sigmoid Activation} \label{sec_sigm_details}
The sigmoid activation~\cite{minaiDerivatives1993} is given by
\ali{
f(y) = \frac{1}{1+ \exp(-y)} \in (0, 1). \label{eq_s_sig_max}
}
Its first, second, and third-order derivatives are expressed as
\ali{
f^\prime(y) &= (1-f(y))f(y), \nn \\
f^{\prime\prime}(y) &= (1-2f(y))(1-f(y))f(y), \nn \\
f^{\prime\prime\prime}(y) &= 6 \Bigg(f(y) - \frac{3+\sqrt{3}}{6}\Bigg)\Bigg(f(y) - \frac{3-\sqrt{3}}{6}\Bigg)\nn \\
&\times (1-f(y))f(y). \nn
}
Since $f^{\prime\prime}(y)=0$ when $f(y) = 1/2$ or $f(y) \rightarrow 0 \lor 1$, the lower bound and maximum of $f^\prime(y)$ are given by
\ali{
\inf_{y \in \mathbb{R}} f^\prime(y) = 0, \ 
\max_{y \in \mathbb{R}} f^\prime(y) =  \frac{1}{4}. \label{eq_sigm_ran_1}
}
Since $f^{\prime\prime\prime}(y)=0$ at $f(y) = (3 \pm \sqrt{3})/6$, the minimum and maximum of $f^{\prime\prime}(y)$ are given by
\ali{
\min_{y \in \mathbb{R}} f^{\prime\prime}(y) &= -\frac{\sqrt{3}}{18} = -0.096..., \nn \\
\max_{y \in \mathbb{R}} f^{\prime\prime}(y) &= \frac{\sqrt{3}}{18} = 0.096... \ .
\label{eq_2nd_sigm_range}
}

\subsection{Hyperbolic Tangent Activation}\label{sec_tanh_details}
The tanh activation~\cite{goodfellowDeep2016} is given by
\ali{
f(y) &= \mathrm{tanh}(y) \in (-1, 1). \label{eq_s_tanh_max}
}
From Eqs.~(4.5.17) and (4.5.73) in~\cite{abramowitzHandbook1965}, the first, second, and third-order derivatives are given by
\ali{
f^\prime(y) &= \mathrm{sech}^2(y) = 1-f(y)^2, \nn \\
f^{\prime\prime}(y) &= -2 f(y) (1-f(y)^2), \nn \\
f^{\prime\prime\prime}(y) &= -2 (1 - f(y)^2) (1-3f(y)^2). \nn
}
Since $f^{\prime}(y)$ approaches 0 as $f(y) \rightarrow -1 \lor 1$, and attains its maximum at $f(y)=0$, we have
\ali{
\inf_{y \in \mathbb{R}} f^{\prime}(y) = 0, \ 
\max_{y \in \mathbb{R}} f^{\prime}(y) = 1. \label{eq_tanh_1st_range}
}
Since the condition $f^{\prime\prime\prime}(y)=0$ is satisfied at $f(y) = \pm 1/\sqrt{3}$, we obtain
\ali{
\min_{y \in \mathbb{R}} f^{\prime\prime}(y) &= -\frac{4\sqrt{3}}{9} = -0.769..., \nn \\ 
\max_{y \in \mathbb{R}} f^{\prime\prime}(y) &= \frac{4\sqrt{3}}{9} = 0.769...\ . \label{eq_tanh_2nd_deri_range}
}

\subsection{SmoothReLU Activation} \label{sec_srelu_details}
The SmoothReLU \cite{berzalDL1012025} is defined as
\ali{
f(y) &= \log (1+ \exp(y)) \in (0, \infty). \label{eq_s_srelu_max}
}
By applying the chain rule, the first-order derivative of SmoothReLU is obtained as
\ali{
f^\prime(y) &= \frac{1}{1+\exp(-y)}. \nn
}
Since this is identical to the sigmoid function, the second-order derivative of SmoothReLU is given by
\ali{
f^{\prime\prime}(y) & = (1-f^\prime(y))f^\prime(y). \nn
}
Furthermore, from the properties of the sigmoid function,
\ali{
&\inf_{y \in \mathbb{R}} f^\prime(y) = 0, \ 
\sup_{y \in \mathbb{R}} f^\prime(y) = 1, \label{eq_srelu_1st_deri_max}\\
&\inf_{y \in \mathbb{R}} f^{\prime\prime}(y) = 0, \ 
\max_{y \in \mathbb{R}} f^{\prime\prime}(y) = 1/4 \label{eq_srelu_2nd_deri_max}
}
holds.

\subsection{GELU activation}\label{sec_gelu_details}
GELU~\cite{hendrycksGaussian2016} is an activation function that employs the cumulative distribution function of the standard Gaussian distribution.
Letting the standard Gaussian distribution and its cumulative distribution function be
\ali{
g(y) &= \frac{1}{\sqrt{2 \pi}} \me^{-\frac{y^{2}}{2}}, \nn \\
G(y) &= \frac{1}{\sqrt{2 \pi}} \int_{-\infty}^{y} \me^{-\frac{a^{2}}{2}} \md a
= \frac{1}{2} ( 1+\mathrm{erf}(y/\sqrt{2}) ), \nn
}
the relation
\ali{
\frac{\md G(y)}{\md y} = g(y), \ \frac{\md g(y)}{\md y} = -yg(y) \nn
}
holds.
In this case, the GELU function is given by
\ali{
f(y) = y G(y), \ \sup_{y \in \mathbb{R}} f(y) = \infty \label{eq_s_gelu_max}
}
as described in~\cite{hendrycksGaussian2016}.
The first, second, and third-order derivatives are given by
\ali{
f^\prime(y) &= G(y) + yg(y), \nn \\
f^{\prime\prime}(y) &= g(y) ( \sqrt{2} - y )( \sqrt{2} + y ), \nn \\
f^{\prime\prime\prime}(y) &= -y g(y) ( 2 -y )( 2 + y ). \nn
}
From $f^{\prime\prime}(y) = 0$, the maximum and minimum values of the first-order derivative are given by
\ali{
\max_{y \in \mathbb{R}} f^\prime(y) &= G(\sqrt{2}) + \sqrt{2}g(\sqrt{2}) = 1.128..., \label{eq_GELU_1st_deri_max} \\
\min_{y \in \mathbb{R}} f^\prime(y) &= G(-\sqrt{2}) -\sqrt{2}g(-\sqrt{2}) = -0.128...\ . \nn
}
From $f^{\prime\prime\prime}(y) = 0$, the maximum and minimum values of the second-order derivative are given by
\ali{
\max_{y \in \mathbb{R}} f^{\prime\prime}(y) &= 2g(0) = \sqrt{\frac{2}{\pi}} = 0.797..., \label{eq_GELU_2nd_deri_max} \\
\min_{y \in \mathbb{R}} f^{\prime\prime}(y) &= g(2) ( \sqrt{2} - 2 )( \sqrt{2} + 2 ) = -0.107... \ .\nn
}

\section{Proofs}
\subsection{Proof for Eq.~\eqref{eq_lambda_equal}} \label{proof_lambda_equal}
From Eqs.~\eqref{eq_main_theorem}, \eqref{eq_ave_eigen}, \eqref{eq_var_eigen}, the relation
\ali{
&\lambda_1 \neq 0 \land \Bigg( \bigwedge_{d=2}^D \lambda_d = 0 \Bigg) \nn \\
& \Rightarrow
\mu(\bm{\theta}) = \frac{\lambda_1}{D}, \ 
\sigma(\bm{\theta})^2 = \frac{\lambda_1^2}{D^2} (D-1) \nn \\
& \Rightarrow \lambda_\mathrm{sup}(\bm{\theta}) = \mu(\bm{\theta}) + \sqrt{D-1} \sigma(\bm{\theta}) = \lambda_1 \nn
}
holds.

\subsection{Proof for Eq.~\eqref{eq_d_lW}} \label{proof_1st_deri_W}
From Eqs.~\eqref{eq_nn_forward} and \eqref{eq_unit_ce}, the path from $\bm{W}$ to $l$ is given by $\bm{W} \rightarrow \bm{y} \rightarrow \bm{r} \rightarrow \bm{h}(\bm{r}) \rightarrow z \rightarrow p \rightarrow l$.
Therefore, the Jacobian matrices between these variables are given by
\ali{
\frac{\partial \bm{y}}{\partial \bm{W}_{:m}} &= x_{m} \bm{E}_N \in \mathbb{R}^{N \times N}, \ 
\frac{\partial \bm{r}}{\partial \bm{y}} = 
\bm{F}^{\prime}(\bm{y}) \in \mathbb{R}^{N \times N}, \nn \\
\frac{\partial \bm{h}(\bm{r})}{\partial \bm{r}} &= \mat{ \bm{0}_N & \bm{E}_N} \in \mathbb{R}^{N \times (N+1)}, \nn \\
\frac{\partial z}{\partial \bm{h}(\bm{r})} &= \bm{V}^\top \in \mathbb{R}^{(N+1) \times 1}, \ 
\frac{\partial p}{\partial z}\frac{\partial l}{\partial p} = p - q. \label{eq_d_pz} 
}
Note that Eq.~\eqref{eq_jacob_sigm} was used for $\partial \bm{r}/\partial \bm{y}$.
In this study, we define the Jacobian matrix as $\partial \bm{\alpha} / \partial \bm{\beta} = (\partial / \partial \bm{\beta})\bm{\alpha}^\top$.
That is, we adopt the denominator layout.
Here, $\bm{E}_N$ denotes the $N \times N$ identity matrix.
Furthermore, we have
\ali{
\frac{\partial p}{\partial z} = s^\prime(z) = p (1-p), \ \frac{\partial l}{\partial p} = \frac{p - q}{p(1-p)}. \label{eq_d_pz_one}
}
By letting the estimation error be
\ali{
\delta = p-q \in (-1, 1), \label{eq_delta_def}
}
we obtain
\ali{
\frac{\partial l}{\partial \bm{W}_{:m}} &= 
\frac{\partial \bm{y}}{\partial \bm{W}_{:m}}
\frac{\partial \bm{r}}{\partial \bm{y}}
\frac{\partial \bm{h}(\bm{r})}{\partial \bm{r}}
\frac{\partial z}{\partial \bm{h}(\bm{r})}
\frac{\partial p}{\partial z}
\frac{\partial l}{\partial p} \nn \\
&= x_{m} \delta \bm{F}^{\prime}(\bm{y})\widetilde{\bm{V}}^\top \in \mathbb{R}^{N \times 1}. \label{eq_d_lWm}
}
From Eq.~\eqref{eq_w_v_flat}, we obtain
\ali{
\frac{\partial l}{\partial \bm{w}} = \mat{\frac{\partial l}{\partial \bm{W}_{:0}} \\ \vdots \\ \frac{\partial l}{\partial \bm{W}_{:M}}} 
= \delta \bm{h}(\bm{x}) \otimes (\bm{F}^{\prime}(\bm{y})\widetilde{\bm{V}}^\top) \in \mathbb{R}^{(M+1)N \times 1}, \nn
}
where $\otimes$ denotes the Kronecker product.

\subsection{Proof for Eq.~\eqref{eq_d_lV}} \label{proof_1st_deri_V}
Since $z = \sum_{n=0}^{N} r_n v_n$ from Eq.~\eqref{eq_nn_forward}, we obtain
\ali{
 \frac{\partial z}{\partial \bm{v}}  = \bm{h}(\bm{r})  \in \mathbb{R}^{(N+1) \times 1}. \label{eq_d_zV}
}
From Eq.~\eqref{eq_nn_forward}, the path from $\bm{V}$ to $l$ is given by $\bm{V} \rightarrow z \rightarrow p \rightarrow l$.
Using Eqs.~\eqref{eq_d_pz} and \eqref{eq_d_zV}, we obtain
\begin{align}
\frac{\partial l}{\partial \bm{v}} = 
 \frac{\partial z}{\partial \bm{v}}
 \frac{\partial p}{\partial z}
 \frac{\partial l}{\partial p} = \delta \bm{h}(\bm{r}) \in \mathbb{R}^{(N+1) \times 1}. \label{eq_lv_app}
\end{align}

\subsection{Proof for Eq.~\eqref{eq_H_th2}} \label{sec_hesse_start}
From Eq.~\eqref{eq_d_ltheta_1}, the Hessian matrix of $l$ is given by
\ali{
& \bm{H}_l(\bm{\theta}, \bm{\theta}) = \frac{\pa }{\pa \bm{\theta}} \bigg( \frac{\pa l}{\pa \bm{\theta}} \bigg)^\top 
= \mat{\frac{\pa }{\pa \bm{w}} \\ \frac{\pa }{\pa \bm{v}}} 
\mat{\big(\frac{\pa l}{\pa \bm{w}}\big)^\top & \big(\frac{\pa l}{\pa \bm{v}}\big)^\top}
\nn \\
&=\mat{
\frac{\pa }{\pa \bm{w}}  \big(\frac{\pa l}{\pa \bm{w}}\big)^\top &
\frac{\pa }{\pa \bm{w}}  \big(\frac{\pa l}{\pa \bm{v}}\big)^\top \\
\frac{\pa }{\pa \bm{v}}  \big(\frac{\pa l}{\pa \bm{w}}\big)^\top &
\frac{\pa }{\pa \bm{v}}  \big(\frac{\pa l}{\pa \bm{v}}\big)^\top \\
}
= \mat{
\bm{H}_l(\bm{w}, \bm{w}) & \bm{H}_l(\bm{w}, \bm{v}) \\
\bm{H}_l(\bm{v}, \bm{w}) & \bm{H}_l(\bm{v}, \bm{v})
}. \nn
}
The following lemma holds.
\begin{lemma}
\ali{
\bm{H}_l(\bm{v}, \bm{w})^\top = \bm{H}_l(\bm{w}, \bm{v}). \label{eq_h_vw_wv_top}
}
\end{lemma}
\begin{proof}
Since higher-order derivatives can be expressed as direct products of operators, we have
\ali{
\bm{H}_l(\bm{v}, \bm{w})^\top
&= \Bigg( \frac{\pa }{\pa \bm{v}}  \bigg(\frac{\pa }{\pa \bm{w}}\bigg)^\top \Bigg)^\top l \nn \\
&= \frac{\pa }{\pa \bm{w}} \bigg(\frac{\pa }{\pa \bm{v}}\bigg)^\top l 
= \bm{H}_l(\bm{w}, \bm{v}). \nn \qedhere
}
\end{proof}

\subsection{Proof for Eq.~\eqref{eq_ww}} \label{sec_proof_HWW}
By the Jacobian chain rule, we obtain
\ali{
 \frac{\partial }{\partial \bm{W}_{:m}} = \frac{\partial \bm{y}}{\partial \bm{W}_{:m}} \frac{\partial }{\partial \bm{y}} = x_{m} \bm{E}_N \frac{\partial }{\partial \bm{y}} = x_{m} \frac{\partial }{\partial \bm{y}}. \label{eq_w2y}
}
Also, applying Eq.~(5) in \cite{petersenMatrix2012} to Eq.~\eqref{eq_d_lWm}, we obtain
\ali{
\bigg( \frac{\partial l}{\partial \bm{W}_{:m}} \bigg)^\top 
&= x_{m}\delta\widetilde{\bm{V}} \bm{F}^\prime(\bm{y}) \in \mathbb{R}^{1 \times N}.
 \label{eq_d_lW_top}
}
Furthermore, from Eqs.~\eqref{eq_d_pz} and \eqref{eq_d_pz_one}, we obtain
\begin{align}
\frac{\partial p}{\partial \bm{y}} &=
\frac{\partial \bm{r}}{\partial \bm{y}}
\frac{\partial \bm{h}(\bm{r})}{\partial \bm{r}}
\frac{\partial z}{\partial \bm{h}(\bm{r})}
\frac{\partial p}{\partial z} = s^\prime(z)\bm{F}^{\prime}(\bm{y})\widetilde{\bm{V}}^\top \in \mathbb{R}^{N \times 1}. \label{eq_py1}
 \end{align}
Combining Eqs.~\eqref{eq_jacob_sigm}, \eqref{eq_w2y}, \eqref{eq_d_lW_top}, and \eqref{eq_py1} with Eq.~(37) in \cite{petersenMatrix2012}, we have
\ali{
&\bm{H}_l(\bm{W}_{:m}, \bm{W}_{:n}) = 
\frac{\partial }{\partial \bm{W}_{:m}} \bigg( \frac{\partial l}{\partial \bm{W}_{:n}} \bigg)^\top 
=x_m x_{n} \frac{\partial }{\partial \bm{y}} \delta \widetilde{\bm{V}} \bm{F}^{\prime}(\bm{y}) \nn\\
&=x_m x_{n}\Bigg( \bigg(\frac{\partial \delta}{\partial \bm{y}} \bigg) \widetilde{\bm{V}} \bm{F}^{\prime}(\bm{y}) \ + \delta \frac{\partial }{\partial \bm{y}} (\widetilde{\bm{V}} \bm{F}^{\prime}(\bm{y}) ) \Bigg)\nn \\
&=x_m x_{n}\Big( s^\prime(z) \bm{F}^{\prime}(\bm{y})\widetilde{\bm{V}}^\top \widetilde{\bm{V}} \bm{F}^{\prime}(\bm{y}) \ + \delta \mathrm{diag}(\widetilde{\bm{V}}^\top) \bm{F}^{\prime\prime}(\bm{y}) \Big), \nn \\
&\bm{H}_l(\bm{W}_{:m}, \bm{W}_{:n}) \in \mathbb{R}^{N \times N}. \nn
}
Therefore,
\ali{
&\bm{H}_l(\bm{w}, \bm{w}) 
= \frac{\pa }{\pa \bm{w}} \bigg( \frac{\pa l}{\pa \bm{w}} \bigg)^\top\nn \\
&=\mat{\frac{\pa }{\pa \bm{W}_{:0}} \\ \vdots \\ \frac{\pa }{\pa \bm{W}_{:M}}} \mat{\big(\frac{\pa l}{\pa \bm{W}_{:0}}\big)^\top & \cdots & \big(\frac{\pa l}{\pa \bm{W}_{:M}}\big)^\top} \nn \\
&=\mat{
\bm{H}_l(\bm{W}_{:0}, \bm{W}_{:0}) & \cdots & \bm{H}_l(\bm{W}_{:0}, \bm{W}_{:M})\\
\vdots & \ddots & \vdots \\
\bm{H}_l(\bm{W}_{:M}, \bm{W}_{:0}) & \cdots & \bm{H}_l(\bm{W}_{:M}, \bm{W}_{:M})}\nn \\
&= \bm{h}(\bm{x})\bm{h}(\bm{x})^\top \nn \\
& \otimes \big( s^\prime(z)\bm{F}^{\prime}(\bm{y})\widetilde{\bm{V}}^\top \widetilde{\bm{V}} \bm{F}^{\prime}(\bm{y}) +\delta \mathrm{diag}(\widetilde{\bm{V}}^\top) \bm{F}^{\prime\prime}(\bm{y}) \big), \nn\\
&\bm{H}_l(\bm{w}, \bm{w}) \in \mathbb{R}^{(M+1)N \times (M+1)N}
\label{eq_H_WW_proof}
}
is satisfied.

\subsection{Proof for Eq.~\eqref{eq_vv}} \label{sec_proof_HVV}
From Eqs.~\eqref{eq_d_pz_one} and \eqref{eq_d_zV}, we obtain
\ali{
\frac{\partial p}{\partial \bm{v}} 
= \frac{\partial z}{\partial \bm{v}}\frac{\partial p}{\partial z}  
= s^\prime(z) \bm{h}(\bm{r}). \label{eq_d_pV}
}
From this, Eq.~\eqref{eq_lv_app}, and Eq.~(37) in \cite{petersenMatrix2012},
\ali{
&\bm{H}_l(\bm{v}, \bm{v})
= \frac{\partial }{\partial \bm{v}}\bigg( \frac{\partial l}{\partial \bm{v}} \bigg)^\top 
= \frac{\partial }{\partial \bm{v}} \delta \bm{h}(\bm{r})^\top  \nn \\
&= \bigg( \frac{\partial \delta}{\partial \bm{v}} \bigg) \bm{h}(\bm{r})^\top +\delta \frac{\partial }{\partial \bm{v}} \bm{h}(\bm{r})^\top 
= \bigg( \frac{\partial p}{\partial \bm{v}} \bigg) \bm{h}(\bm{r})^\top \nn \\
&=  s^\prime(z) \bm{h}(\bm{r}) \bm{h}(\bm{r})^\top \in \mathbb{R}^{(N+1) \times (N+1)} \label{eq_H_lVV}
}
is satisfied.

\subsection{Proofs for Eqs.~\eqref{eq_vw} and \eqref{eq_wv}}\label{sec_proof_HVW}
From Eqs.~\eqref{eq_d_lW_top}, \eqref{eq_d_pV}, and Eq.~(37) in \cite{petersenMatrix2012}, we obtain
\ali{
& \bm{H}_l(\bm{v}, \bm{W}_{:m})
= \frac{\partial }{\partial \bm{v}} \bigg( \frac{\partial l}{\partial \bm{W}_{:m}} \bigg)^\top 
= x_m\frac{\partial }{\partial \bm{v}} \delta \widetilde{\bm{V}} \bm{F}^{\prime}(\bm{y})\nn \\
& = x_m \bigg( 
\bigg( \frac{\partial \delta}{\partial \bm{v}} \bigg)  \widetilde{\bm{V}} \bm{F}^{\prime}(\bm{y}) +\delta\frac{\partial }{\partial \bm{v}}  \widetilde{\bm{V}} \bm{F}^{\prime}(\bm{y})
\bigg)\nn \\
& = x_m\bigg( 
s^\prime(z) \bm{h}(\bm{r})  \widetilde{\bm{V}} \bm{F}^{\prime}(\bm{y}) +\delta\mat{\bm{0}_N^\top \\ \bm{F}^{\prime}(\bm{y})}
\bigg) \in \mathbb{R}^{(N+1) \times N}. \nn 
 }
Therefore, this yields
\ali{
& \bm{H}_l(\bm{v}, \bm{w}) = \frac{\partial }{\partial \bm{v}} \bigg( \frac{\partial l}{\partial \bm{w}} \bigg)^\top\nn\\
 &= \mat{\frac{\partial }{\partial \bm{v}} \big(\frac{\partial l}{\partial \bm{W}_{:0}}\big)^\top & \cdots & 
 \frac{\partial }{\partial \bm{v}}\big(\frac{\partial l}{\partial \bm{W}_{:M}}\big)^\top } \nn \\
  &= \mat{\bm{H}_l(\bm{v}, \bm{W}_{:0}) & \cdots & \bm{H}_l(\bm{v}, \bm{W}_{:M}) } \nn\\
 &=\bm{h}(\bm{x})^\top \otimes  \bigg( 
s^\prime(z) \bm{h}(\bm{r})  \widetilde{\bm{V}} \bm{F}^{\prime}(\bm{y}) +\delta\mat{\bm{0}_N^\top \\ \bm{F}^{\prime}(\bm{y})}
\bigg), \nn \\
&\bm{H}_l(\bm{v}, \bm{w})  \in \mathbb{R}^{(N+1) \times (M+1)N}. \nn
}
From Eq.~\eqref{eq_h_vw_wv_top}, we have
\ali{
& \bm{H}_l(\bm{w}, \bm{v}) = \bm{H}_l(\bm{v}, \bm{w})^\top \nn \\
& =  \bm{h}(\bm{x}) \otimes   
\Big( s^\prime(z) \bm{F}^{\prime}(\bm{y}) \widetilde{\bm{V}}^\top \bm{h}(\bm{r})^\top+\delta \mat{\bm{0}_N & \bm{F}^{\prime}(\bm{y})} \Big), \nn \\
& \bm{H}_l(\bm{w}, \bm{v}) \in \mathbb{R}^{(M+1)N \times (N+1)}, \nn
}
where Eqs.~(4), (5), and (510) in \cite{petersenMatrix2012} have been applied.

\subsection{Proof for Eqs.~\eqref{eq_H_l_the2_known_all_ww} and \eqref{eq_H_l_the2_known_all_vv}}\label{sec_proof_trace}
Based on Eq.~\eqref{eq_H_lVV} and Eq.~(14) in \cite{petersenMatrix2012}, we obtain
\ali{
&\mathrm{tr}(\bm{H}_l(\bm{v}, \bm{v})) = s^\prime(z) \mathrm{tr}(\bm{h}(\bm{r}) \bm{h}(\bm{r})^\top)
=s^\prime(z) \mathrm{tr}(\bm{h}(\bm{r})^\top \bm{h}(\bm{r}) )\nn \\
&=s^\prime(z) \bm{h}(\bm{r})^\top \bm{h}(\bm{r}) = s^\prime(z)(1+\bm{r}^\top\bm{r}). \label{eq_h_tr_vv}
}
Employing Eq.~\eqref{eq_H_WW_proof} and Eq.~(515) in \cite{petersenMatrix2012} leads to
\ali{
&\mathrm{tr}(\bm{H}_l(\bm{w}, \bm{w})) = \mathrm{tr}(\bm{h}(\bm{x})\bm{h}(\bm{x})^\top \nn \\
&\otimes ( s^\prime(z)\bm{F}^{\prime}(\bm{y})\widetilde{\bm{V}}^\top \widetilde{\bm{V}} \bm{F}^{\prime}(\bm{y}) + \delta \mathrm{diag}(\widetilde{\bm{V}}^\top) \bm{F}^{\prime\prime}(\bm{y})) ) \nn\\
&= (1+\bm{x}^\top \bm{x})\mathrm{tr}(
s^\prime(z)\bm{F}^{\prime}(\bm{y})\widetilde{\bm{V}}^\top \widetilde{\bm{V}} \bm{F}^{\prime}(\bm{y}) + \delta \mathrm{diag}(\widetilde{\bm{V}}^\top) \bm{F}^{\prime\prime}(\bm{y})
) \nn \\
&= 
(1+\bm{x}^\top \bm{x}) ( s^\prime(z) \| \bm{F}^{\prime}(\bm{y}) \widetilde{\bm{V}}^\top\|^2 + \delta \widetilde{\bm{V}} \bm{f}^{\prime\prime}(\bm{y}) ).
 \label{eq_h_tr_ww}
}
Furthermore, from Eq.~(15) in~\cite{petersenMatrix2012}, the trace for all training data is expressed as
\ali{
\mathrm{tr}(\bm{H}_L(\bm{a}, \bm{a})) &= \sum_{i=1}^{I} \mathrm{tr}(\bm{H}_{l_i}(\bm{a}, \bm{a})), \ \bm{a} \in \{\bm{w}, \bm{v}\}. \label{eq_h_vv_ww}
}
Substituting
\ali{
(\bm{E}_{I} + \bm{X}^\top \bm{X} \odot \bm{E}_I)\bm{1}_I &=  
 \mat{
 1+\bm{x}_i^\top \bm{x}_i
 } \in \mathbb{R}^{I \times 1}, \nn \\
 (\bm{E}_{I} + \bm{R}^\top \bm{R} \odot \bm{E}_I)\bm{1}_I &=  
 \mat{
 1+\bm{r}_i^\top \bm{r}_i
 } \in \mathbb{R}^{I \times 1}, \nn \\
\bm{s}^\prime \odot \big( \bm{F}^{\prime\odot 2}_{\mathrm{y}} \widetilde{\bm{V}}^{\odot 2 \top} \big) 
&= \bm{s}^\prime \odot \big( (\bm{F}^\prime_\mathrm{y} \odot \bm{F}^\prime_\mathrm{y}) (\widetilde{\bm{V}} \odot \widetilde{\bm{V}})^\top \big)\nn \\
&= \mat{
s^\prime(z_i) \| \bm{F}^\prime(\bm{y}_i) \widetilde{\bm{V}}^\top\|^2 
} \in \mathbb{R}^{I \times 1}, \nn\\
\bm{\delta} \odot (\bm{F}^{\prime\prime}_{\mathrm{y}} \widetilde{\bm{V}}^\top) &= \mat{
\delta_i \bm{f}^{\prime\prime}(\bm{y}_i)^\top \widetilde{\bm{V}}^\top}\nn \\
&= \mat{
\delta_i \widetilde{\bm{V}} \bm{f}^{\prime\prime}(\bm{y}_i)
} \in \mathbb{R}^{I \times 1} \nn
}
into Eqs.~\eqref{eq_H_l_the2_known_all_ww} and \eqref{eq_H_l_the2_known_all_vv} and simplifying the terms shows that Eq.~\eqref{eq_h_vv_ww} and Eqs.~\eqref{eq_H_l_the2_known_all_ww}--\eqref{eq_H_l_the2_known_all_vv} are equivalent.

\subsection{Proof for Eq.~\eqref{eq_tr_ub_vv}}\label{sec_proof_tr_ub_vv}
From Eq.~\eqref{eq_h_tr_vv}, we obtain
$\mathrm{tr}(\bm{H}_{l}(\bm{v}, \bm{v})) = s^\prime(z)(1+\bm{r}^\top\bm{r})$.
Since $s(z)$ is a function that outputs the estimated probability in $(0, 1)$ according to Eq.~\eqref{eq_3nn}, it follows that
\ali{
\max_{z \in \mathbb{R}} s^\prime(z) = \max_{p \in (0, 1)} p(1-p) = \frac{1}{4}. \label{eq_pz_1st_up}
}
It holds from Eqs.~\eqref{eq_s_lin_max}, \eqref{eq_s_sig_max}, \eqref{eq_s_tanh_max}, \eqref{eq_s_srelu_max}, and \eqref{eq_s_gelu_max} that
\ali{
&\sup_{\bm{r}} \ (1+\bm{r}^\top\bm{r}) = 
\sup_{\bm{y} \in \mathbb{R}^N} (1+\bm{f}(\bm{y})^\top\bm{f}(\bm{y}))\nn\\
&=
\begin{cases}
N+1, f: \mathrm{Sigmoid/Tanh} \\
\infty, f: \mathrm{Linear/SmoothReLU/GELU} 
\end{cases}. \nn
}
Therefore, we have
\ali{
\mathrm{tr}(\bm{H}_{l}(\bm{v}, \bm{v})) < 
\begin{cases}
\frac{1}{4} (N+1), f: \mathrm{Sigmoid/Tanh} \\
\infty, f: \mathrm{Linear/SmoothReLU/GELU} 
\end{cases}. \nn
}
Substituting this into Eq.~\eqref{eq_h_vv_ww} yields Eq.~\eqref{eq_tr_ub_vv}.

\subsection{Proof for Eq.~\eqref{eq_tr_ub_ww}}\label{sec_proof_tr_ub_ww}
Based on Eq.~\eqref{eq_h_tr_ww}, we independently derive the upper bounds for $s^\prime(z) \| \bm{F}^{\prime}(\bm{y}) \widetilde{\bm{V}}^\top\|^2$ and $\delta \widetilde{\bm{V}} \bm{f}^{\prime\prime}(\bm{y})$.
From Eqs.~\eqref{eq_s_lin_1st_max}, \eqref{eq_sigm_ran_1}, \eqref{eq_tanh_1st_range}, \eqref{eq_srelu_1st_deri_max}, \eqref{eq_GELU_1st_deri_max}, and \eqref{eq_pz_1st_up}, we obtain
\ali{
&\sup_{z \in \mathbb{R}, \ \bm{y} \in \mathbb{R}^N} s^\prime(z) \| \bm{F}^{\prime}(\bm{y}) \widetilde{\bm{V}}^\top\|^2
= \frac{1}{4} \zeta_1 \| \widetilde{\bm{V}}^\top \|^2 = \zeta_1 \| \widetilde{\bm{V}} \|^2_\mathrm{F}, \nn \\
&\zeta_1 = \begin{cases}
1, & f: \mathrm{Linear/Tanh/SmoothReLU}\\
\frac{1}{16}, & f: \mathrm{Sigmoid}\\
1.272..., & f: \mathrm{GELU}
\end{cases} \nn
}
where we have used $(1.128...)^2 = 1.272...$ for the GELU case.
Combining Eqs.~\eqref{eq_s_lin_2nd_max}, \eqref{eq_2nd_sigm_range}, \eqref{eq_tanh_2nd_deri_range}, \eqref{eq_srelu_2nd_deri_max}, \eqref{eq_GELU_2nd_deri_max}, and \eqref{eq_delta_def} yields
\ali{
&\sup_{\delta \in (-1, 1), \ \bm{y} \in \mathbb{R}^N} \delta \widetilde{\bm{V}} \bm{f}^{\prime\prime}(\bm{y})
= \zeta_2 | \widetilde{\bm{V}}\bm{1}_N | \nn, \\
&\zeta_2 = \begin{cases}
0, & f: \mathrm{Linear}\\
\frac{\sqrt{3}}{18} = 0.096..., & f: \mathrm{Sigmoid}\\
\frac{4\sqrt{3}}{9} = 0.769..., & f: \mathrm{Tanh}\\
\frac{1}{4}, & f: \mathrm{SmoothReLU} \\
\sqrt{\frac{2}{\pi}} = 0.797..., & f: \mathrm{GELU}
\end{cases}.\nn
}
Therefore, we have
\ali{
\mathrm{tr}(\bm{H}_l(\bm{w}, \bm{w})) \leq 
 \bigg( \frac{1}{4} \zeta_1 \| \widetilde{\bm{V}} \|^2_\mathrm{F} + \zeta_2 | \widetilde{\bm{V}}\bm{1}_N | \bigg) (1+\bm{x}^\top \bm{x}).\nn
}
Consequently, we obtain
\ali{
&\mathrm{tr}(\bm{H}_L(\bm{w}, \bm{w})) \leq 
 \bigg( \frac{1}{4} \zeta_1 \| \widetilde{\bm{V}} \|^2_\mathrm{F} + \zeta_2 | \widetilde{\bm{V}}\bm{1}_N | \bigg) \sum_{i=1}^I (1+\bm{x}_i^\top \bm{x}_i)\nn \\
 &= 
 \bigg( \frac{1}{4} \zeta_1 \| \widetilde{\bm{V}} \|^2_\mathrm{F} + \zeta_2 | \widetilde{\bm{V}}\bm{1}_N | \bigg) \bm{1}_I^\top (\bm{E}_I + \bm{X}^\top \bm{X} \odot \bm{E}_I) \bm{1}_I. \nn
}

\subsection{Proof for Eq.~\eqref{eq_max_largex}}\label{proof_max_largex}
When each component of $\bm{x}$ is normalized to a maximum of $1$, the relation
\ali{
& \max_{\bm{X} \in [0, 1]^{M \times I}} \bm{1}_I^\top (\bm{E}_I + \bm{X}^\top \bm{X} \odot \bm{E}_I) \bm{1}_I\nn\\
&= \bm{1}_I^\top (\bm{E}_I + ((\bm{1}_M \bm{1}_I^\top)^\top \bm{1}_M \bm{1}_I^\top) \odot \bm{E}_I) \bm{1}_I\nn\\
&= \bm{1}_I^\top (\bm{E}_I + (\bm{1}_I \bm{1}_M^\top \bm{1}_M \bm{1}_I^\top) \odot \bm{E}_I) \bm{1}_I\nn\\
&= \bm{1}_I^\top (\bm{E}_I + M(\bm{1}_I \bm{1}_I^\top) \odot \bm{E}_I) \bm{1}_I\nn\\
&= (1+M) \bm{1}_I^\top \bm{E}_I \bm{1}_I = I(1+M) \nn
}
holds, where we have used the fact that the expression is maximized when $\bm{X} = \bm{1}_M \bm{1}_I^\top$.

\subsection{Proof for Eq.~\eqref{eq_h2_tr_fin}: Step 1}\label{sec_proof_h2_prop_step1}
In this section, we derive $\mathrm{tr} ( \bm{H}_L(\bm{\theta}, \bm{\theta})^2 )$.
The following lemma holds.
\begin{lemma}
\ali{
\mathrm{tr} ( \bm{H}_L(\bm{\theta}, \bm{\theta})^2 ) &= \sum_{i=1}^I  \sum_{j=1}^I \mathrm{tr}( \bm{H}_{l_i}(\bm{w}, \bm{w}) \bm{H}_{l_j}(\bm{w}, \bm{w})) \nn\\
&+2\sum_{i=1}^I  \sum_{j=1}^I \mathrm{tr}( \bm{H}_{l_i}(\bm{w}, \bm{v}) \bm{H}_{l_j}(\bm{v}, \bm{w}) ) \nn \\
&+ \sum_{i=1}^I  \sum_{j=1}^I \mathrm{tr} (\bm{H}_{l_i}(\bm{v}, \bm{v}) \bm{H}_{l_j}(\bm{v}, \bm{v})). \label{eq_sum_hh}
}
\end{lemma}
\begin{proof}
From Eq.~\eqref{eq_hessian_L}, the squared Hessian is given by
\ali{
 \bm{H}_L(\bm{\theta}, \bm{\theta})^2 &= \Bigg( \sum_{i=1}^I  \bm{H}_{l_i}(\bm{\theta}, \bm{\theta}) \Bigg)^2 \nn \\
 &= \sum_{i=1}^I  \sum_{j=1}^I \bm{H}_{l_i}(\bm{\theta}, \bm{\theta}) \bm{H}_{l_j}(\bm{\theta}, \bm{\theta}) \in \mathbb{R}^{D \times D}. \nn
}
From Eq.~\eqref{eq_H_th2}, since $\bm{H}_{l}(\bm{\theta}, \bm{\theta})$ is a $2\times 2$ block matrix, $\bm{H}_{l_i}(\bm{\theta}, \bm{\theta}) \bm{H}_{l_j}(\bm{\theta}, \bm{\theta})$ is also a $2\times 2$ block matrix.
If we denote the diagonal blocks as $(\bm{H}_{l_i}(\bm{\theta}, \bm{\theta}) \bm{H}_{l_j}(\bm{\theta}, \bm{\theta}))_{11}$ and $(\bm{H}_{l_i}(\bm{\theta}, \bm{\theta}) \bm{H}_{l_j}(\bm{\theta}, \bm{\theta}))_{22}$, the trace of the squared Hessian can be written, according to Eq.~15 of~\cite{petersenMatrix2012}, as
\ali{
\mathrm{tr}( \bm{H}_L(\bm{\theta}, \bm{\theta})^2 ) 
& = \sum_{i=1}^I  \sum_{j=1}^I \mathrm{tr}( (\bm{H}_{l_i}(\bm{\theta}, \bm{\theta}) \bm{H}_{l_j}(\bm{\theta}, \bm{\theta}))_{11} )  \nn \\
& + \sum_{i=1}^I  \sum_{j=1}^I \mathrm{tr}((\bm{H}_{l_i}(\bm{\theta}, \bm{\theta}) \bm{H}_{l_j}(\bm{\theta}, \bm{\theta}))_{22}). \label{eq_h2_tr_proof_sec}
}
From Eq.~\eqref{eq_H_th2}, the block components related to the trace can be written as
\ali{
&(\bm{H}_{l_i}(\bm{\theta}, \bm{\theta}) \bm{H}_{l_j}(\bm{\theta}, \bm{\theta}))_{11}  \nn \\
&=\bm{H}_{l_i}(\bm{w}, \bm{w}) \bm{H}_{l_j}(\bm{w}, \bm{w})+\bm{H}_{l_i}(\bm{w}, \bm{v}) \bm{H}_{l_j}(\bm{v}, \bm{w}), \nn\\
&(\bm{H}_{l_i}(\bm{\theta}, \bm{\theta}) \bm{H}_{l_j}(\bm{\theta}, \bm{\theta}))_{22} \nn \\
&= \bm{H}_{l_i}(\bm{v}, \bm{w}) \bm{H}_{l_j}(\bm{w}, \bm{v})+\bm{H}_{l_i}(\bm{v}, \bm{v}) \bm{H}_{l_j}(\bm{v}, \bm{v}). \nn
}
Their sizes are given by
\ali{
&(\bm{H}_{l_i}(\bm{\theta}, \bm{\theta}) \bm{H}_{l_j}(\bm{\theta}, \bm{\theta}))_{11} 
\in \mathbb{R}^{(M+1)N \times (M+1)N}, \nn \\ 
&(\bm{H}_{l_i}(\bm{\theta}, \bm{\theta}) \bm{H}_{l_j}(\bm{\theta}, \bm{\theta}))_{22} 
\in \mathbb{R}^{(N+1) \times (N+1)}. \nn
}
Furthermore, from Eq.~\eqref{eq_h_vw_wv_top} and Eqs.~(13) and (14) of~\cite{petersenMatrix2012}, we have
\ali{
\mathrm{tr} ( \bm{H}_{l_i}(\bm{v}, \bm{w}) \bm{H}_{l_j}(\bm{w}, \bm{v}) ) 
&= \mathrm{tr} ( \bm{H}_{l_j}(\bm{w}, \bm{v}) \bm{H}_{l_i}(\bm{v}, \bm{w}) )\nn \\
&=\mathrm{tr} ( \bm{H}_{l_j}(\bm{v}, \bm{w})^\top \bm{H}_{l_i}(\bm{w}, \bm{v})^\top )\nn \\
&=\mathrm{tr} (  ( \bm{H}_{l_i}(\bm{w}, \bm{v}) \bm{H}_{l_j}(\bm{v}, \bm{w}))^\top )\nn \\
&=\mathrm{tr} (  \bm{H}_{l_i}(\bm{w}, \bm{v}) \bm{H}_{l_j}(\bm{v}, \bm{w}) ). \nn
}
Thus, Eq.~\eqref{eq_h2_tr_proof_sec} can be expressed as Eq.~\eqref{eq_sum_hh}.
\end{proof}
In order to find the trace of the squared Hessian, it is sufficient to find
\ali{
&\mathrm{tr}( \bm{H}_{l_i}(\bm{w}, \bm{w}) \bm{H}_{l_j}(\bm{w}, \bm{w})), \label{eq_hh_wwww_ap} \\
&\mathrm{tr}( \bm{H}_{l_i}(\bm{w}, \bm{v}) \bm{H}_{l_j}(\bm{v}, \bm{w})), \label{eq_hh_wvvw_ap} \\ 
&\mathrm{tr}(\bm{H}_{l_i}(\bm{v}, \bm{v}) \bm{H}_{l_j}(\bm{v}, \bm{v})). \label{eq_hh_vvvv_ap}
}
In this equation, the trace of a Kronecker product appears.
This can be processed more efficiently by the following.
\begin{lemma}
\ali{
\mathrm{tr}((\bm{A}_1 \otimes \bm{A}_2)(\bm{B}_1 \otimes \bm{B}_2)) = \mathrm{tr}(\bm{A}_1 \bm{B}_1) \mathrm{tr}(\bm{A}_2 \bm{B}_2) \label{eq_krock_trace}
}
\end{lemma}
\begin{proof}
From Eq.~(511) of \cite{petersenMatrix2012}, $(\bm{A}_1\otimes \bm{A}_2)(\bm{B}_1 \otimes \bm{B}_2) = (\bm{A}_1 \bm{B}_1) \otimes (\bm{A}_2 \bm{B}_2)$ holds.
From Eq.~(515) of \cite{petersenMatrix2012}, it holds that $\mathrm{tr}(\bm{C}_1 \otimes \bm{C}_2) = \mathrm{tr}(\bm{C}_1) \mathrm{tr}(\bm{C}_2)$.
Therefore, we obtain Eq.~\eqref{eq_krock_trace}.
\end{proof}

\subsection{Proof for Eq.~\eqref{eq_h2_tr_fin}: Step 2}\label{sec_proof_h2_prop_step2}
Eq.~\eqref{eq_hh_wwww_ap} is given as follows.
\begin{lemma}
\ali{
&\mathrm{tr} ( \bm{H}_{l_i}(\bm{w}, \bm{w}) \bm{H}_{l_j}(\bm{w}, \bm{w}) ) \nn \\
&= (1+\bm{x}_i^\top \bm{x}_j)^2
\Big( s^\prime (z_i) s^\prime (z_j) (\widetilde{\bm{V}} \bm{F}^\prime(\bm{y}_j) \bm{F}^\prime(\bm{y}_i) \widetilde{\bm{V}}^\top)^2 \nn \\
&+s^\prime (z_i) \delta_j \widetilde{\bm{V}}  \bm{F}^\prime(\bm{y}_i) \mathrm{diag}(\widetilde{\bm{V}}^\top)\bm{F}^{\prime\prime}(\bm{y}_j) \bm{F}^\prime(\bm{y}_i) \widetilde{\bm{V}}^\top\nn \\
&+s^\prime (z_j) \delta_i \widetilde{\bm{V}} \bm{F}^\prime(\bm{y}_j) \mathrm{diag}(\widetilde{\bm{V}}^\top)\bm{F}^{\prime\prime}(\bm{y}_i) \bm{F}^\prime(\bm{y}_j) \widetilde{\bm{V}}^\top \nn \\
&+\delta_i \delta_j \widetilde{\bm{V}} \bm{F}^{\prime\prime}(\bm{y}_i) \bm{F}^{\prime\prime}(\bm{y}_j) \widetilde{\bm{V}}^\top\Big). \label{eq_wwww}
}
\end{lemma}
\begin{proof}
Using Eq.~\eqref{eq_ww}, we have
\ali{
\bm{H}_{l_i}(\bm{w}, \bm{w}) = \bm{A}_1 \otimes \bm{A}_2, \ 
\bm{H}_{l_j}(\bm{w}, \bm{w}) = \bm{B}_1 \otimes \bm{B}_2 \nn
}
and thus, from Eq.~\eqref{eq_krock_trace}, it can be written as
\ali{
&\mathrm{tr} ( \bm{H}_{l_i}(\bm{w}, \bm{w}) \bm{H}_{l_j}(\bm{w}, \bm{w})) = \mathrm{tr}(\bm{A}_1 \bm{B}_1) \mathrm{tr}(\bm{A}_2 \bm{B}_2). \label{eq_wwww_AB}
}
By using Eq.~(16) of~\cite{petersenMatrix2012}, we obtain
\ali{
\mathrm{tr}(\bm{A}_1 \bm{B}_1) &= \mathrm{tr}(\bm{h}(\bm{x}_i)\bm{h}(\bm{x}_i)^\top \bm{h}(\bm{x}_j)\bm{h}(\bm{x}_j)^\top) \nn \\
&= \mathrm{tr}( \bm{h}(\bm{x}_j)^\top \bm{h}(\bm{x}_i) \bm{h}(\bm{x}_i)^\top \bm{h}(\bm{x}_j))  \nn \\
&= (\bm{h}(\bm{x}_i)^\top \bm{h}(\bm{x}_j))^2 = (1+\bm{x}_i^\top \bm{x}_j)^2. \label{eq_wwww_A1B1} 
}
Since $\bm{A}_2 \bm{B}_2$ is
\ali{
&\bm{A}_2 \bm{B}_2 \nn \\
&= s^\prime (z_i) s^\prime (z_j)\bm{F}^\prime(\bm{y}_i) \widetilde{\bm{V}}^\top \widetilde{\bm{V}} \bm{F}^\prime(\bm{y}_i) \bm{F}^\prime(\bm{y}_j) \widetilde{\bm{V}}^\top \widetilde{\bm{V}} \bm{F}^\prime(\bm{y}_j)\nn \\
&+s^\prime (z_i) \delta_j \bm{F}^\prime(\bm{y}_i) \widetilde{\bm{V}}^\top \widetilde{\bm{V}} \bm{F}^\prime(\bm{y}_i) \mathrm{diag}(\widetilde{\bm{V}}^\top) \bm{F}^{\prime\prime}(\bm{y}_j)\nn \\
&+s^\prime (z_j) \delta_i \mathrm{diag}(\widetilde{\bm{V}}^\top) \bm{F}^{\prime\prime}(\bm{y}_i)
\bm{F}^\prime(\bm{y}_j) \widetilde{\bm{V}}^\top \widetilde{\bm{V}} \bm{F}^\prime(\bm{y}_j) \nn \\
&+\delta_i \delta_j \mathrm{diag}(\widetilde{\bm{V}}^\top) \bm{F}^{\prime\prime}(\bm{y}_i) \mathrm{diag}(\widetilde{\bm{V}}^\top)\bm{F}^{\prime\prime}(\bm{y}_j), \nn
}
we focus on the traces of the individual terms.
These are shown by using the cyclic property of the trace.
By applying the cyclic property of the trace twice in the first term and noting that the products of diagonal matrices are commutative, the trace becomes a scalar.
As a result, the trace of the first term becomes
\ali{
&\mathrm{tr}([\bm{A}_2 \bm{B}_2]_1)\nn \\
&=s^\prime (z_i) s^\prime (z_j)\mathrm{tr}( \bm{F}^\prime(\bm{y}_i) \widetilde{\bm{V}}^\top \widetilde{\bm{V}} \bm{F}^\prime(\bm{y}_i) \bm{F}^\prime(\bm{y}_j) \widetilde{\bm{V}}^\top \widetilde{\bm{V}} \bm{F}^\prime(\bm{y}_j) )\nn \\
&=s^\prime (z_i) s^\prime (z_j)\mathrm{tr}( \widetilde{\bm{V}} \bm{F}^\prime(\bm{y}_j) \bm{F}^\prime(\bm{y}_i) \widetilde{\bm{V}}^\top \widetilde{\bm{V}} \bm{F}^\prime(\bm{y}_i) \bm{F}^\prime(\bm{y}_j) \widetilde{\bm{V}}^\top )\nn \\
&=s^\prime (z_i) s^\prime (z_j) (\widetilde{\bm{V}} \bm{F}^\prime(\bm{y}_j) \bm{F}^\prime(\bm{y}_i) \widetilde{\bm{V}}^\top)^2. \label{eq_wwww_A2B2T1}
}
Here, we have used the commutativity of diagonal matrices, i.e., $\bm{F}^\prime(\bm{y}_i) \bm{F}^\prime(\bm{y}_j) = \bm{F}^\prime(\bm{y}_j) \bm{F}^\prime(\bm{y}_i)$.
Next, we find $\mathrm{tr}([\bm{A}_2 \bm{B}_2]_2)$.
By cyclically shifting the expression inside the trace in the second term four times to the right, it becomes a scalar.
By doing so, we obtain
\ali{
&\mathrm{tr}([\bm{A}_2 \bm{B}_2]_2)\nn \\
&=s^\prime (z_i) \delta_j \mathrm{tr}( \bm{F}^\prime(\bm{y}_i) \widetilde{\bm{V}}^\top \widetilde{\bm{V}} \bm{F}^\prime(\bm{y}_i)
\mathrm{diag}(\widetilde{\bm{V}}^\top)\bm{F}^{\prime\prime}(\bm{y}_j) )\nn \\
&=s^\prime (z_i) \delta_j \widetilde{\bm{V}}  \bm{F}^\prime(\bm{y}_i) \mathrm{diag}(\widetilde{\bm{V}}^\top) \bm{F}^{\prime\prime}(\bm{y}_j) \bm{F}^\prime(\bm{y}_i) \widetilde{\bm{V}}^\top. \label{eq_wwww_A2B2T2}
}
Next, we evaluate $\mathrm{tr}([\bm{A}_2 \bm{B}_2]_3)$.
Since the argument of the trace in the third term becomes a scalar by cyclically shifting it twice to the right, we obtain
\ali{
&\mathrm{tr}([\bm{A}_2 \bm{B}_2]_3) \nn \\
&=s^\prime (z_j) \delta_i \mathrm{tr}( \mathrm{diag}(\widetilde{\bm{V}}^\top)\bm{F}^{\prime\prime}(\bm{y}_i)
\bm{F}^\prime(\bm{y}_j) \widetilde{\bm{V}}^\top \widetilde{\bm{V}} \bm{F}^\prime(\bm{y}_j)
)\nn \\
&=s^\prime (z_j) \delta_i \widetilde{\bm{V}} \bm{F}^\prime(\bm{y}_j) \mathrm{diag}(\widetilde{\bm{V}}^\top) \bm{F}^{\prime\prime}(\bm{y}_i) \bm{F}^\prime(\bm{y}_j) \widetilde{\bm{V}}^\top. \label{eq_wwww_A2B2T3}
}
Next, we evaluate $\mathrm{tr}([\bm{A}_2 \bm{B}_2]_4)$.
We use the fact that for real-valued vectors $\bm{a}_1$ and $\bm{a}_2$ of the same dimension,
\ali{
\mathrm{tr}( \mathrm{diag}(\bm{a}_1)  \mathrm{diag}(\bm{a}_2) ) = 
\mathrm{tr}( \mathrm{diag}(\bm{a}_1 \odot \bm{a}_2) )
 = \bm{a}_1^\top \bm{a}_2 \label{eq_diag_tr}
}
holds.
Using this, we obtain
\ali{
&\mathrm{tr}([\bm{A}_2 \bm{B}_2]_4) \nn \\
&=\delta_i \delta_j \mathrm{tr}(
\mathrm{diag}(\widetilde{\bm{V}}^\top)\bm{F}^{\prime\prime}(\bm{y}_i)
\mathrm{diag}(\widetilde{\bm{V}}^\top)\bm{F}^{\prime\prime}(\bm{y}_j)
)\nn \\
&=\delta_i \delta_j \mathrm{tr}\big(
\mathrm{diag}(\widetilde{\bm{V}}^\top \odot \bm{f}^{\prime\prime}(\bm{y}_i)) 
\mathrm{diag}(\widetilde{\bm{V}}^\top \odot \bm{f}^{\prime\prime}(\bm{y}_j)) 
\big)\nn \\
&=\delta_i \delta_j (\widetilde{\bm{V}}^\top \odot \bm{f}^{\prime\prime}(\bm{y}_i))^\top (\widetilde{\bm{V}}^\top \odot \bm{f}^{\prime\prime}(\bm{y}_j))\nn \\
&=\delta_i \delta_j \widetilde{\bm{V}} \bm{F}^{\prime\prime}(\bm{y}_i) \bm{F}^{\prime\prime}(\bm{y}_j) \widetilde{\bm{V}}^\top. \label{eq_wwww_A2B2T4}
}
Substituting Eqs.~\eqref{eq_wwww_A1B1}, \eqref{eq_wwww_A2B2T1}, \eqref{eq_wwww_A2B2T2}, \eqref{eq_wwww_A2B2T3}, and \eqref{eq_wwww_A2B2T4} into Eq.~\eqref{eq_wwww_AB} yields Eq.~\eqref{eq_wwww}.
\end{proof}

\subsection{Proof for Eq.~\eqref{eq_h2_tr_fin}: Step 3}\label{sec_proof_h2_prop_step3}
Eq.~\eqref{eq_hh_wvvw_ap} can be expressed as follows.
\begin{lemma}
\ali{
&\mathrm{tr}( \bm{H}_{l_i}(\bm{w}, \bm{v}) \bm{H}_{l_j}(\bm{v}, \bm{w}))
=(1+\bm{x}_i^\top \bm{x}_j) \times \nn \\
&\Big( s^\prime (z_i) s^\prime (z_j) (1+\bm{r}_i^\top \bm{r}_j) \widetilde{\bm{V}} \bm{F}^\prime(\bm{y}_j) \bm{F}^\prime(\bm{y}_i) \widetilde{\bm{V}}^\top \nn \\
&+s^\prime (z_i) \delta_j \bm{r}_i^\top \bm{F}^\prime(\bm{y}_j) \bm{F}^\prime(\bm{y}_i) \widetilde{\bm{V}}^\top \nn \\
&+s^\prime (z_j) \delta_i \widetilde{\bm{V}} \bm{F}^\prime(\bm{y}_j) \bm{F}^\prime(\bm{y}_i) \bm{r}_j\nn \\
&+\delta_i\delta_j \bm{f}^\prime(\bm{y}_i)^\top \bm{f}^\prime(\bm{y}_j)
\Big). \label{eq_wvvw}
}
\end{lemma}
\begin{proof}
From Eqs.~\eqref{eq_vw} and \eqref{eq_wv}, by letting
\ali{
\bm{H}_{l_i}(\bm{w}, \bm{v}) = \bm{A}_1 \otimes \bm{A}_2, \ 
\bm{H}_{l_j}(\bm{v}, \bm{w}) = \bm{B}_1 \otimes \bm{B}_2, \nn
}
it follows from Eq.~\eqref{eq_krock_trace} that
\ali{
&\mathrm{tr}( \bm{H}_{l_i}(\bm{w}, \bm{v}) \bm{H}_{l_j}(\bm{v}, \bm{w})) = \mathrm{tr}(\bm{A}_1 \bm{B}_1) \mathrm{tr}(\bm{A}_2 \bm{B}_2) \label{eq_wvvw_AB}
}
holds.
From Eq.~(16) in~\cite{petersenMatrix2012}, we obtain
\ali{
\mathrm{tr}(\bm{A}_1\bm{B}_1) &= \mathrm{tr}(\bm{h}(\bm{x}_i) \bm{h}(\bm{x}_j)^\top) 
= \mathrm{tr}(\bm{h}(\bm{x}_j)^\top \bm{h}(\bm{x}_i)) \nn \\
&= \bm{h}(\bm{x}_j)^\top \bm{h}(\bm{x}_i)
= 1+\bm{x}_i^\top \bm{x}_j. \label{eq_wvvw_A1B1}
}
Since $\bm{A}_2\bm{B}_2$ can be expressed as
\ali{
\bm{A}_2\bm{B}_2 &= 
\big( s^\prime(z_i) \bm{F}^\prime(\bm{y}_i) \widetilde{\bm{V}}^\top \bm{h}(\bm{r}_i)^\top  
+\delta_i \mat{\bm{0}_N & \bm{F}^\prime(\bm{y}_i)} \big)\nn \\
&\times \bigg( 
s^\prime(z_j) \bm{h}(\bm{r}_j)  \widetilde{\bm{V}} \bm{F}^\prime(\bm{y}_j)
+\delta_j \mat{\bm{0}_N^\top \\ \bm{F}^\prime(\bm{y}_j)} \bigg)
\nn \\
&=s^\prime(z_i)s^\prime(z_j) \bm{F}^\prime(\bm{y}_i) \widetilde{\bm{V}}^\top \bm{h}(\bm{r}_i)^\top \bm{h}(\bm{r}_j)  \widetilde{\bm{V}} \bm{F}^\prime(\bm{y}_j)\nn \\
&+s^\prime(z_i)\delta_j \bm{F}^\prime(\bm{y}_i) \widetilde{\bm{V}}^\top \bm{h}(\bm{r}_i)^\top 
\mat{\bm{0}_N^\top \\ \bm{F}^\prime(\bm{y}_j)}\nn \\
&+s^\prime(z_j)\delta_i \mat{\bm{0}_N & \bm{F}^\prime(\bm{y}_i)} 
\bm{h}(\bm{r}_j)  \widetilde{\bm{V}} \bm{F}^\prime(\bm{y}_j)\nn \\
&+\delta_i\delta_j \mat{\bm{0}_N & \bm{F}^\prime(\bm{y}_i)}
\mat{\bm{0}_N^\top \\ \bm{F}^\prime(\bm{y}_j)}, \nn
}
we evaluate the trace of each term individually.
By noting that the expression $\bm{F}^\prime(\bm{y}_i) \widetilde{\bm{V}}^\top \widetilde{\bm{V}} \bm{F}^\prime(\bm{y}_j)$ becomes a scalar after cyclically shifting it twice to the right, it follows that
\ali{
&\mathrm{tr}([\bm{A}_2\bm{B}_2]_1) \nn \\
&= s^\prime(z_i)s^\prime(z_j) \mathrm{tr}( \bm{F}^\prime(\bm{y}_i) \widetilde{\bm{V}}^\top \bm{h}(\bm{r}_i)^\top \bm{h}(\bm{r}_j)  \widetilde{\bm{V}} \bm{F}^\prime(\bm{y}_j) )\nn \\
&= s^\prime(z_i)s^\prime(z_j) \bm{h}(\bm{r}_i)^\top \bm{h}(\bm{r}_j) \mathrm{tr}( \bm{F}^\prime(\bm{y}_i) \widetilde{\bm{V}}^\top  \widetilde{\bm{V}} \bm{F}^\prime(\bm{y}_j) )\nn \\
&= s^\prime(z_i)s^\prime(z_j) \bm{h}(\bm{r}_i)^\top \bm{h}(\bm{r}_j) \mathrm{tr}( \widetilde{\bm{V}} \bm{F}^\prime(\bm{y}_j) \bm{F}^\prime(\bm{y}_i) \widetilde{\bm{V}}^\top  )\nn \\
&= s^\prime(z_i)s^\prime(z_j) (1+\bm{r}_i^\top \bm{r}_j) \widetilde{\bm{V}} \bm{F}^\prime(\bm{y}_j) \bm{F}^\prime(\bm{y}_i) \widetilde{\bm{V}}^\top. \label{eq_wvvw_A2B2T1}
}
Regarding $\mathrm{tr}([\bm{A}_2\bm{B}_2]_2)$, noting that the expression $\bm{F}^\prime(\bm{y}_i) \widetilde{\bm{V}}^\top \bm{r}_i^\top \bm{F}^\prime(\bm{y}_j)$ becomes a scalar after cyclically shifting it twice to the right, we obtain
\ali{
\mathrm{tr}([\bm{A}_2\bm{B}_2]_2) &= 
s^\prime(z_i)\delta_j \mathrm{tr}\bigg( \bm{F}^\prime(\bm{y}_i) \widetilde{\bm{V}}^\top \bm{h}(\bm{r}_i)^\top 
\mat{\bm{0}_N^\top \\ \bm{F}^\prime(\bm{y}_j)} \bigg)\nn \\
&= 
s^\prime(z_i)\delta_j \mathrm{tr}\bigg( \bm{F}^\prime(\bm{y}_i) \widetilde{\bm{V}}^\top 
\mat{1 & \bm{r}_i^\top}
\mat{\bm{0}_N^\top \\ \bm{F}^\prime(\bm{y}_j)} \bigg) \nn \\
&= 
s^\prime(z_i)\delta_j \mathrm{tr} ( \bm{F}^\prime(\bm{y}_i) \widetilde{\bm{V}}^\top 
\bm{r}_i^\top \bm{F}^\prime(\bm{y}_j) )\nn \\
&= 
s^\prime(z_i)\delta_j \mathrm{tr}( \bm{r}_i^\top \bm{F}^\prime(\bm{y}_j) \bm{F}^\prime(\bm{y}_i) \widetilde{\bm{V}}^\top )\nn \\
&= 
s^\prime(z_i)\delta_j \bm{r}_i^\top \bm{F}^\prime(\bm{y}_j) \bm{F}^\prime(\bm{y}_i) \widetilde{\bm{V}}^\top.
\label{eq_wvvw_A2B2T2}
}
Similarly, $\mathrm{tr}([\bm{A}_2\bm{B}_2]_3)$ is given by
\ali{
\mathrm{tr}([\bm{A}_2\bm{B}_2]_3) &=
s^\prime(z_j) \delta_i \mathrm{tr}( \mat{\bm{0}_N & \bm{F}^\prime(\bm{y}_i)} 
\bm{h}(\bm{r}_j)  \widetilde{\bm{V}} \bm{F}^\prime(\bm{y}_j))\nn \\
&=
s^\prime(z_j) \delta_i \mathrm{tr}\bigg( \mat{\bm{0}_N & \bm{F}^\prime(\bm{y}_i)} 
\mat{1 \\ \bm{r}_j}
\widetilde{\bm{V}} \bm{F}^\prime(\bm{y}_j)\bigg)\nn \\
&=
s^\prime(z_j) \delta_i \mathrm{tr}( \bm{F}^\prime(\bm{y}_i) \bm{r}_j
\widetilde{\bm{V}} \bm{F}^\prime(\bm{y}_j) )\nn \\
&=
s^\prime(z_j) \delta_i \mathrm{tr}( \widetilde{\bm{V}} \bm{F}^\prime(\bm{y}_j) \bm{F}^\prime(\bm{y}_i) \bm{r}_j
 )\nn \\
&=
s^\prime(z_j) \delta_i \widetilde{\bm{V}} \bm{F}^\prime(\bm{y}_j) \bm{F}^\prime(\bm{y}_i) \bm{r}_j.\label{eq_wvvw_A2B2T3}
 }
Using Eq.~\eqref{eq_diag_tr}, $\mathrm{tr}([\bm{A}_2\bm{B}_2]_4)$ becomes
\ali{
\mathrm{tr}([\bm{A}_2\bm{B}_2]_4) &=
\delta_i \delta_j \mathrm{tr}\bigg( \mat{\bm{0}_N & \bm{F}^\prime(\bm{y}_i)}
\mat{\bm{0}_N^\top \\ \bm{F}^\prime(\bm{y}_j)} \bigg)\nn \\
&=
\delta_i \delta_j \mathrm{tr}( \bm{F}^\prime(\bm{y}_i)
 \bm{F}^\prime(\bm{y}_j) )\nn \\
&=\delta_i \delta_j \bm{f}^\prime(\bm{y}_i)^\top \bm{f}^\prime(\bm{y}_j). \label{eq_wvvw_A2B2T4}
}
Substituting Eqs.~\eqref{eq_wvvw_A1B1}, \eqref{eq_wvvw_A2B2T1}, \eqref{eq_wvvw_A2B2T2}, \eqref{eq_wvvw_A2B2T3}, and \eqref{eq_wvvw_A2B2T4} into Eq.~\eqref{eq_wvvw_AB} yields Eq.~\eqref{eq_wvvw}.
\end{proof}

\subsection{Proof for Eq.~\eqref{eq_h2_tr_fin}: Step 4}\label{sec_proof_h2_prop_step4}
Eq.~\eqref{eq_hh_vvvv_ap} is given as follows.
\begin{lemma}
\ali{
\mathrm{tr}( \bm{H}_{l_i}(\bm{v}, \bm{v})\bm{H}_{l_j}(\bm{v}, \bm{v}) ) 
= s^\prime(z_i)s^\prime(z_j) (1+\bm{r}_i^\top \bm{r}_j)^2. \label{eq_vvvv}
}
\end{lemma}
\begin{proof}
From Eq.~\eqref{eq_vv}, using Eq.~(16) in \cite{petersenMatrix2012} and cyclically shifting the expression inside the trace to obtain a scalar, we have
\ali{
&\mathrm{tr}( \bm{H}_{l_i}(\bm{v}, \bm{v})\bm{H}_{l_j}(\bm{v}, \bm{v}) ) \nn \\
&= s^\prime(z_i) s^\prime(z_j) \mathrm{tr}( \bm{h}(\bm{r}_i) \bm{h}(\bm{r}_i)^\top \bm{h}(\bm{r}_j) \bm{h}(\bm{r}_j)^\top) \nn \\ 
&= s^\prime(z_i) s^\prime(z_j) \mathrm{tr}(\bm{h}(\bm{r}_j)^\top \bm{h}(\bm{r}_i) \bm{h}(\bm{r}_i)^\top \bm{h}(\bm{r}_j)) \nn \\
&= s^\prime(z_i) s^\prime(z_j) ( \bm{h}(\bm{r}_i)^\top \bm{h}(\bm{r}_j) )^2\nn\\
&= s^\prime(z_i) s^\prime(z_j) ( 1+\bm{r}_i^\top \bm{r}_j )^2. \nn \qedhere
}
\end{proof}

\subsection{Proof for Eq.~\eqref{eq_h2_tr_fin}: Step 5}\label{sec_proof_h2_prop_step5}
Using Eqs.~\eqref{eq_wwww}, \eqref{eq_wvvw}, and \eqref{eq_vvvv}, an upper bound of $\mathrm{tr} ( \bm{H}_L(\bm{\theta}, \bm{\theta})^2 )$ is given as follows.
\begin{proposition}
\ali{
\mathrm{tr} ( \bm{H}_L(\bm{\theta}, \bm{\theta})^2 )
&= \sum_{i=1}^I \sum_{j=1}^I \phi_{ij} (1+ \bm{x}_i^\top\bm{x}_j)^2 \nn \\
&+ 2\sum_{i=1}^I \sum_{j=1}^I \psi_{ij} (1+ \bm{x}_i^\top\bm{x}_j)\nn \\
&+\sum_{i=1}^I \sum_{j=1}^I \omega_{ij}(1+ \bm{r}_i^\top\bm{r}_j)^2 \geq 0. \label{eq_h2_tr_fin_scc}
}
\end{proposition}
\begin{proof}
Eqs.~\eqref{eq_wwww}, \eqref{eq_wvvw}, and \eqref{eq_vvvv} are expressed as
\ali{
\mathrm{tr} ( \bm{H}_{l_i}(\bm{w}, \bm{w}) \bm{H}_{l_j}(\bm{w}, \bm{w}) ) &= \phi_{ij} (1+ \bm{x}_i^\top\bm{x}_j)^2, \label{eq_h2tr_ww_phi} \\
\mathrm{tr}( \bm{H}_{l_i}(\bm{w}, \bm{v}) \bm{H}_{l_j}(\bm{v}, \bm{w})) &= \psi_{ij} (1+ \bm{x}_i^\top\bm{x}_j), \nn \\
\mathrm{tr}( \bm{H}_{l_i}(\bm{v}, \bm{v}) \bm{H}_{l_j}(\bm{v}, \bm{v})) &= \omega_{ij}(1+ \bm{r}_i^\top\bm{r}_j)^2, \nn
}
where we used Eqs.~\eqref{eq_phi_ij}, \eqref{eq_psi_ij}, and \eqref{eq_ome_ij}.
Substituting these into Eq.~\eqref{eq_sum_hh} yields Eq.~\eqref{eq_h2_tr_fin_scc}.
\end{proof}

Furthermore, since the Frobenius inner product for $\bm{A}, \bm{B} \in \mathbb{R}^{I \times I}$ is defined as $\langle \bm{A}, \bm{B} \rangle_\mathrm{F} = \bm{1}_I^\top (\bm{A} \odot \bm{B}) \bm{1}_I = \sum_{i, j=1}^I a_{ij}b_{ij}$, we obtain
\ali{
\langle \bm{\Phi}, (\bm{J}_I + \bm{X}^\top \bm{X})^{\odot 2} \rangle_\mathrm{F} &= \bm{1}_I^\top (\bm{\Phi} \odot (\bm{J}_I + \bm{X}^\top \bm{X})^{\odot 2}) \bm{1}_I\nn \\
&= \bm{1}_I^\top \mat{\phi_{ij} (1+\bm{x}_i^\top \bm{x}_j)^2}_{i, j \in \mathbb{N}_{\leq I}} \bm{1}_I\nn \\
&= \sum_{i=1}^I \sum_{j=1}^I \phi_{ij} (1+ \bm{x}_i^\top\bm{x}_j)^2, \nn}
\ali{\langle \bm{\Psi}, (\bm{J}_I + \bm{X}^\top \bm{X}) \rangle_\mathrm{F} 
&=  \bm{1}_I^\top (\bm{\Psi} \odot (\bm{J}_I + \bm{X}^\top \bm{X})) \bm{1}_I\nn \\
&= \bm{1}_I^\top \mat{\psi_{ij} (1+\bm{x}_i^\top \bm{x}_j)}_{i, j \in \mathbb{N}_{\leq I}} \bm{1}_I\nn \\
&=\sum_{i=1}^I \sum_{j=1}^I \psi_{ij} (1+ \bm{x}_i^\top\bm{x}_j), \nn}
\ali{\langle \bm{\Omega}, (\bm{J}_I + \bm{R}^\top \bm{R})^{\odot 2} \rangle_\mathrm{F}
&= \bm{1}_I^\top (\bm{\Omega} \odot (\bm{J}_I + \bm{R}^\top \bm{R})^{\odot 2}) \bm{1}_I\nn \\
&= \bm{1}_I^\top \mat{\omega_{ij} (1+\bm{r}_i^\top \bm{r}_j)^2}_{i, j \in \mathbb{N}_{\leq I}} \bm{1}_I\nn \\
&= \sum_{i=1}^I \sum_{j=1}^I \omega_{ij} (1+ \bm{r}_i^\top\bm{r}_j)^2. \nn
}
Therefore, Eq.~\eqref{eq_h2_tr_fin_scc} and Eq.~\eqref{eq_h2_tr_fin} are equivalent.

\subsection{Proofs for Eqs.~\eqref{eq_h2tr_upper_mat} and \eqref{eq_phi_max}} \label{sec_proof_2htr_upper_v2}
We decompose $\mathrm{tr} ( \bm{H}_L(\bm{\theta}, \bm{\theta})^2 )$ in Eq.~\eqref{eq_h2_tr_fin_scc} into three terms and derive an upper bound for each.
\begin{lemma}
\ali{
[\mathrm{tr} ( \bm{H}_L(\bm{\theta}, \bm{\theta})^2 )]_1 
\leq \phi_\mathrm{max} \sum_{i=1}^I \sum_{j=1}^I (1+\bm{x}_i^\top \bm{x}_j)^2. \label{eq_h2tr_up_T1}
}
\end{lemma}
\begin{proof}
From Eq.~\eqref{eq_h2tr_ww_phi}, since the trace of a squared matrix is equal to the sum of the squares of its eigenvalues, for $i=j$, we have
\ali{
&\forall i \in \mathbb{N}_{\leq I}, \mathrm{tr}( \bm{H}_{l_i}(\bm{w}, \bm{w})^2)
=\phi_{ii}(1+ \bm{x}_i^\top \bm{x}_i)^2 \geq 0\nn \\
& \therefore \forall i \in \mathbb{N}_{\leq I}, \phi_{ii} \geq 0 \Rightarrow \phi_\mathrm{max}  \geq 0. \nn
}
Therefore, from Eq.~\eqref{eq_h2_tr_fin_scc}, we obtain
\ali{
[\mathrm{tr} ( \bm{H}_L(\bm{\theta}, \bm{\theta})^2 )]_1 
&= \sum_{i=1}^I \sum_{j=1}^I \phi_{ij} (1+ \bm{x}_i^\top \bm{x}_j)^2\nn \\
&\leq \phi_\mathrm{max} \sum_{i=1}^I \sum_{j=1}^I (1+ \bm{x}_i^\top \bm{x}_j)^2. \nn \qedhere
}
\end{proof}
\begin{lemma}
\ali{
[\mathrm{tr} ( \bm{H}_L(\bm{\theta}, \bm{\theta})^2 )]_2 
\leq \sum_{i=1}^I \sum_{j=1}^I (\psi_{ij})^2 + \sum_{i=1}^I \sum_{j=1}^I (1+ \bm{x}_i^\top \bm{x}_j)^2. \label{eq_h2tr_up_T2}
}
\end{lemma}
\begin{proof}
From Eq.~\eqref{eq_h2_tr_fin_scc}, using the inequality of arithmetic and geometric means, we obtain
\ali{
&[\mathrm{tr} ( \bm{H}_L(\bm{\theta}, \bm{\theta})^2 )]_2 = 2\sum_{i=1}^I \sum_{j=1}^I \psi_{ij} (1+ \bm{x}_i^\top \bm{x}_j) \nn \\
&\leq \sum_{i=1}^I \sum_{j=1}^I (\psi_{ij})^2 + \sum_{i=1}^I \sum_{j=1}^I (1+ \bm{x}_i^\top \bm{x}_j)^2. \nn \qedhere
}
\end{proof}
\begin{lemma}
\ali{
[\mathrm{tr} ( \bm{H}_L(\bm{\theta}, \bm{\theta})^2 )]_3 
\leq \frac{1}{16} \sum_{i=1}^I \sum_{j=1}^I (1+ \bm{r}_i^\top \bm{r}_j)^2. \label{eq_h2tr_up_T3}
}
\end{lemma}
\begin{proof}
From Eqs.~\eqref{eq_ome_ij}, \eqref{eq_pz_1st_up}, and \eqref{eq_h2_tr_fin_scc}, we have
\ali{
&\omega_{ij} = s^\prime(z_i)s^\prime(z_j), \ \therefore  0 < \omega_{ij} \leq 1/16 , \ \forall i, j \in \mathbb{N}_{\leq I} \nn \\
&\Rightarrow [\mathrm{tr} ( \bm{H}_L(\bm{\theta}, \bm{\theta})^2 )]_3 
= \sum_{i=1}^I \sum_{j=1}^I \omega_{ij} (1+ \bm{r}_i^\top \bm{r}_j)^2\nn \\
&\leq \frac{1}{16} \sum_{i=1}^I \sum_{j=1}^I (1+ \bm{r}_i^\top \bm{r}_j)^2. \nn \qedhere
}
\end{proof}

Summing the upper bounds of the first to third terms in Eqs.~\eqref{eq_h2tr_up_T1}, \eqref{eq_h2tr_up_T2}, and \eqref{eq_h2tr_up_T3}, we obtain the following upper bound for $\mathrm{tr} ( \bm{H}_L(\bm{\theta}, \bm{\theta})^2 )$.
\begin{proposition}
\ali{
&0 \leq \mathrm{tr} ( \bm{H}_L(\bm{\theta}, \bm{\theta})^2 ) 
\leq \sum_{i=1}^I \sum_{j=1}^I (\psi_{ij})^2 + \nn \\
& ( 1 + \phi_\mathrm{max} ) \sum_{i=1}^I \sum_{j=1}^I (1+\bm{x}_i^\top\bm{x}_j)^2
+ \frac{1}{16} \sum_{i=1}^I \sum_{j=1}^I (1+\bm{r}_i^\top\bm{r}_j)^2. \label{eq_h2tr_upper}
}
\end{proposition}
The coefficients in Eq.~\eqref{eq_h2tr_upper} can be expressed in terms of the Frobenius norm. 
By using $\bm{J}_I = \bm{1}_I \bm{1}_I^\top$, we obtain
\ali{
\| \bm{\Psi} \|_\mathrm{F}^2 &= \sum_{i=1}^I \sum_{j=1}^I (\psi_{ij})^2, \nn \\
\| \bm{J}_I + \bm{X}^\top \bm{X} \|_\mathrm{F}^2 &= \sum_{i=1}^I \sum_{j=1}^I (1+\bm{x}_i^\top\bm{x}_j)^2, \nn \\
\| \bm{J}_I + \bm{R}^\top \bm{R} \|_\mathrm{F}^2 &= \sum_{i=1}^I \sum_{j=1}^I (1+\bm{r}_i^\top\bm{r}_j)^2. \nn
}
Substituting these relations into Eq.~\eqref{eq_h2tr_upper} yields Eq.~\eqref{eq_h2tr_upper_mat}.

\subsection{Proofs for Eqs.~\eqref{eq_jxx_cond1} and \eqref{eq_jrr_cond1}}\label{pr_jxx_cond1}
When the input data is normalized, the maximum value of $\| \bm{J}_I + \bm{X}^\top \bm{X} \|_\mathrm{F}^2$ is given by
\ali{
&\max_{\bm{X} \in [0, 1]^{M \times I}} \| \bm{J}_I + \bm{X}^\top \bm{X} \|_\mathrm{F}^2 
= \| \bm{J}_I + (\bm{1}_M \bm{1}_I^\top)^\top \bm{1}_M \bm{1}_I^\top \|_\mathrm{F}^2 \nn \\
& = \| (1+M) \bm{J}_I \|_\mathrm{F}^2 =  I^2 (1+M)^2. \nn
}
For saturating activation functions, the upper bound of $\| \bm{J}_I + \bm{R}^\top \bm{R} \|_\mathrm{F}^2$ is derived as
\ali{
&f: \mathrm{Sigmoid} \lor \mathrm{Tanh} \Rightarrow \sup_{y \in \mathbb{R}} f(y) = 1\nn \\
&\Rightarrow \sup_{\bm{R}} \| \bm{J}_I + \bm{R}^\top \bm{R} \|_\mathrm{F}^2 
= \| \bm{J}_I + (\bm{1}_N \bm{1}_I^\top)^\top \bm{1}_N \bm{1}_I^\top \|_\mathrm{F}^2\nn \\
&= \| (1+N) \bm{J}_I \|_\mathrm{F}^2 = I^2 (1+N)^2. \nn 
}

\subsection{Proof for Eq.~\eqref{eq_limit_u_eigenmax}} \label{proof_limit_u_eigenmax}
From Eqs.~\eqref{eq_H_l_the2_known_all_ww} and \eqref{eq_H_l_the2_known_all_vv}, it follows that
\ali{
&\bm{\delta} \rightarrow \bm{0}_I \Rightarrow s(z_i) \rightarrow 0 \lor 1, \forall i \in \mathbb{N}_{\leq I} 
\Rightarrow \bm{s}^\prime \rightarrow \bm{0}_I  \nn \\
& \Rightarrow \mathrm{tr}(\bm{H}_L(\bm{w}, \bm{w})) \rightarrow 0 \land \mathrm{tr}(\bm{H}_L(\bm{v}, \bm{v})) \rightarrow 0  \nn \\
& \Rightarrow \mathrm{tr}(\bm{H}_L(\bm{\theta}, \bm{\theta})) \rightarrow 0. \nn
}
Furthermore, from Eq.~\eqref{eq_h2_tr_fin}, we obtain
\ali{
&\bm{\delta} \rightarrow \bm{0}_I \Rightarrow 
\forall i \in \mathbb{N}_{\leq I}, 
(s(z_i) \rightarrow 0 \lor 1 \Rightarrow \bm{o}_i \rightarrow \bm{0}_2) \nn \\
& \Rightarrow \bm{\Phi} \rightarrow \bm{0}_{I \times I} \land 
\bm{\Psi} \rightarrow \bm{0}_{I \times I} \land \bm{\Omega} \rightarrow \bm{0}_{I \times I}  \nn \\
& \Rightarrow \mathrm{tr}(\bm{H}_L(\bm{\theta}, \bm{\theta})^2) \rightarrow 0. \nn
}
From Eqs.~\eqref{eq_main_theorem}, \eqref{eq_ave_eigen}, and \eqref{eq_var_eigen}, we have
\ali{
\bm{\delta} \rightarrow \bm{0}_I
\Rightarrow \mu(\bm{\theta}) \rightarrow 0 \land \sigma(\bm{\theta}) \rightarrow 0 
\Rightarrow \lambda_\mathrm{sup}(\bm{\theta}) \rightarrow 0.\nn
}

\subsection{Proof for Eq.~\eqref{eq_btw_tr_tr2}} \label{sec_proof_btw_tr_tr2}
From the Cauchy-Schwarz inequality, we obtain
\ali{
&(\bm{\lambda}^\top \bm{1}_D)^2 \leq \| \bm{\lambda} \|^2 \| \bm{1}_D \|^2 \nn \\
&\Leftrightarrow \Bigg( \sum_{d=1}^D \lambda_d \Bigg)^2 \leq 
\sum_{d=1}^D \lambda_d^2
 \sum_{d=1}^D 1^2 \nn \\
&\Leftrightarrow \frac{1}{D}\mathrm{tr} ( \bm{H}_L(\bm{\theta}, \bm{\theta}) )^2 \leq
\mathrm{tr} ( \bm{H}_L(\bm{\theta}, \bm{\theta})^2 ).  \label{eq_tr_tr2_low}
}
Additionally, we have
\ali{
\mathrm{tr} ( \bm{H}_L(\bm{\theta}, \bm{\theta}) )^2
&= \Bigg( \sum_{d=1}^D \lambda_d \Bigg)^2
= \sum_{d=1}^D \lambda_d^2 + \sum_{d_1 \neq d_2} \lambda_{d_1}\lambda_{d_2}\nn \\
&= \mathrm{tr} ( \bm{H}_L(\bm{\theta}, \bm{\theta})^2 ) + \sum_{d_1 \neq d_2} \lambda_{d_1}\lambda_{d_2}. \nn
}
If the Hessian matrix is positive semi-definite, we obtain $\sum_{d_1 \neq d_2} \lambda_{d_1}\lambda_{d_2} \geq 0$.
Therefore, 
\ali{
\bm{H}_L(\bm{\theta}, \bm{\theta}) \succeq 0 \Rightarrow
 \mathrm{tr} ( \bm{H}_L(\bm{\theta}, \bm{\theta})^2 ) \leq \mathrm{tr} ( \bm{H}_L(\bm{\theta}, \bm{\theta}) )^2 \label{eq_tr_tr2_up}
 }
holds.
From Eqs.~\eqref{eq_tr_tr2_low} and \eqref{eq_tr_tr2_up}, Eq.~\eqref{eq_btw_tr_tr2} holds.

\section{Experiments} \label{sec_expe}
To bridge the theory and experiments, we constructed a three-layer sigmoid NN ($M=2, N=3$), where $M$ and $N$ are the dimensions of input $\bm{x}$ and hidden layer $\bm{r}$, respectively.
The task is a binary classification problem to determine which of two Gaussian distributions generated the 2D data $\bm{x}$.
Regarding the distribution parameters, Class 0 has a mean of $\mat{1 & 1}^\top$, and Class 1 has a mean of $\mat{-1 & -1}^\top$.
This classification problem has also been adopted in previous studies investigating the eigenspectrum of the Hessian~\cite{sagunEigenvalues2017}.
The variance along both the vertical and horizontal axes was set to 2, with zero covariance.
These settings yield the scatter plots shown in Figs.~\ref{fig_eigen_max_high} and \ref{fig_eigen_max_low}.
The training and test datasets consist of $I = 50$ and $10^3$ samples, respectively.
Each class contains half of these samples, ensuring a balanced dataset.

Initial parameters $\bm{\theta} = \bm{\theta}_0$ were generated using 500 random seeds and updated via the gradient descent rule
\ali{
\bm{\theta}_{t+1} = \bm{\theta}_{t} - \gamma \frac{\pa L}{\pa \bm{\theta}}\bigg|_{\bm{\theta}=\bm{\theta}_{t}}, \ t \in \mathbb{N}_{< T} \cup \{0\}, \nn
}
where $T$ denotes the maximum number of iterations.
The purpose of this experiment is not to construct a model with high generalization performance, but to investigate the properties of sharpness at various critical points.
To this end, it is necessary to observe a diverse range of critical points. 
For this reason, we sampled $\bm{\theta}_0$ from a uniform distribution over the wide open interval $(-10, 10)^D$.
The convergence to a critical point was determined using a small threshold $\epsilon$ such that
\ali{
\Bigg\| \frac{\pa L}{\pa \bm{\theta}}\bigg|_{\bm{\theta}=\bm{\theta}_{t}}  \Bigg\|
 \Bigg( \Bigg\| \frac{\pa L}{\pa \bm{\theta}}\bigg|_{\bm{\theta}=\bm{\theta}_{0}} \Bigg\| \Bigg)^{-1} < \epsilon. \nn
}
In the experiments, we used $T=10^4$ and $\epsilon = 0.001$.
To avoid the inappropriate inclusion of redundant critical points, we retained only one instance of any duplicate.
We calculated the Euclidean distances between all pairs of critical points and defined those with a distance less than $\mathrm{Ave.} - 3 \times \mathrm{Std.}$ as duplicates.

As a result of the procedure described above, 460 out of the 500 initial parameters converged to critical points, resulting in 353 unique critical points after removing duplicates.
Similar procedures were applied to experiments where the hidden layer dimension $N$ was increased from 3 to 10.
In this case, 453 out of the 500 initial parameters reached critical points, resulting in 420 unique critical points after removing duplicates.
These critical points were made the subjects of analysis in this study.


\bibliographystyle{elsarticle-num}
\bibliography{refs}

\end{document}